%% file: consideration-aistats.tex
\documentclass[twoside]{article}

\usepackage[accepted]{aistats2025}
\usepackage[round,sort]{natbib}

\usepackage{amsmath,amsthm,amssymb}
\usepackage{hyperref}       % hyperlinks
\usepackage{url}   
\usepackage{cleveref}
\usepackage{thm-restate}
\newtheorem{theorem}{Theorem}[section]
\newtheorem{lemma}[theorem]{Lemma}

\usepackage{algorithm}
\usepackage[noend]{algorithmic}
\usepackage{booktabs}
\usepackage{csvsimple}
\usepackage{tikz}
\usepackage{ifthen}

\usetikzlibrary{positioning,calc,fit,shapes}
\pgfdeclarelayer{fg}
\pgfsetlayers{main,fg}

\DeclareMathOperator{\E}{E}
\def\set#1{\ensuremath{\left\{#1\right\}}}
\def\tup#1{\ensuremath{\left\langle#1\right\rangle}}
\let\intersect=\cap
\def\cap{\mathop{\intersect}}
\let\eps=\varepsilon
\newcommand{\U}{\mathcal U}
\newcommand{\R}{\mathcal R}     
\newcommand{\PL}{{\textstyle \Pr_{\textup{\textsc{\scriptsize PL}}}}}
\newcommand{\PLC}{{\textstyle \Pr_{\textup{\textsc{\scriptsize PL+C}}}}}
\newcommand{\PrC}{{\textstyle \Pr_{\textup{\textsc{\scriptsize C}}}}}

\makeatletter
\renewcommand{\AISTATS@appearing}{Authors' version of a paper appearing in the Proceedings of the \@conferenceordinal\,International Conference on Artificial
  Intelligence and Statistics (AISTATS) \@conferenceyear,  \@conferencelocation\@. PMLR: Volume  \@conferencevolume. Copyright
  \@conferenceyear\/ by the authors.}
\makeatother

\begin{document}



\twocolumn[

\aistatstitle{When the Universe is Too Big: Bounding Consideration Probabilities for Plackett--Luce Rankings}

\aistatsauthor{ Ben Aoki-Sherwood \And Catherine Bregou \And David Liben-Nowell }

\aistatsaddress{University of Colorado Boulder \And Carleton College \And Carleton College } 

\aistatsauthor{Kiran Tomlinson \And Thomas Zeng} 

\aistatsaddress{Microsoft Research \And University of Wisconsin--Madison}]
\runningauthor{Aoki-Sherwood, Bregou, Liben-Nowell, Tomlinson, Zeng}

\begin{abstract}
The widely used Plackett--Luce ranking model assumes that individuals rank items by making repeated choices from a universe of items. But in many cases the universe is too big for people to plausibly consider all options. In the choice literature, this issue has been addressed by supposing that individuals first sample a small consideration set and then choose among the considered items. However, inferring unobserved consideration sets (or item consideration probabilities) in this ``consider then choose'' setting poses significant challenges, because even simple models of consideration with strong independence assumptions are not identifiable, even if item utilities are known.
  We apply the consider-then-choose framework to top-$k$ rankings, where we assume rankings are constructed according to a Plackett--Luce model after sampling a consideration set.
  While item consideration probabilities remain non-identified in this setting, we prove that we can infer bounds on the relative values of consideration probabilities. 
  Additionally, given a condition on the expected consideration set size and known item utilities, we derive absolute upper and lower bounds on item consideration probabilities.
We also provide algorithms to tighten those bounds on consideration probabilities by propagating inferred constraints.
  Thus, we show that we can learn useful information about consideration probabilities despite not being able to identify them precisely.
  We demonstrate our methods on a ranking dataset from a psychology experiment with two different ranking tasks (one with fixed consideration sets and one with unknown consideration sets).
  This combination of data allows us to estimate utilities and then learn about unknown consideration probabilities using our bounds. 
\end{abstract}

\section{INTRODUCTION}

Among the wide-ranging topics studied in the behavioral sciences, predicting and explaining human choices is a central challenge across a range of disciplines.
(Why \emph{that} entr\'{e}e at \emph{that} caf\'{e} with \emph{that} person, of all the restaurants you can think of, of all your potential dates, of all the dishes on the menu?)
Settings where individuals select an item from a collection of available alternatives are well studied in the literature on \emph{discrete choice}~\citep{train2009discrete}, with applications ranging from marketing~\citep{chintagunta2011structural} and voting~\citep{thurner2000empirical} to transportation policy~\citep{horne2005improving} and recommender systems~\citep{danaf2019online}.
A closely related line of work studies \emph{rankings}~\citep{alvo2014statistical}, rather than single choices, often modeling a ranking as a sequence of discrete choices (selecting the top-ranked item, then the second, etc.).
The Plackett--Luce ranking model~\citep{plackett1975analysis,luce1959individual} exemplifies the link between discrete choice and ranking, positing that items at each of $k$ positions are selected in turn according to a logit choice model~\citep{mcfadden1973conditional}. Due to its convex negative log-likelihood and easily interpretable parameters, the Plackett--Luce model has seen substantial use in the machine learning literature~\citep{seshadri2020learning,zhao2016learning,zhao2019learning,nguyen2023efficient,saha2020pac}.

However, as Plackett--Luce models are applied to datasets of increasing size, it becomes implausible that individuals are able to weigh all of their options against each other. This issue has been largely ignored in the ranking setting, but has been discussed at length in the discrete choice literature.
For instance, in keeping with the framework of bounded rationality~\citep{simon1957models}, a prominent line of work in discrete choice suggests that selection is a two-stage process, where individuals first narrow their options to a small \emph{consideration set} from which their final selection is made~\citep{hauser1990evaluation,shocker1991consideration}.
These so-called ``consider then choose'' models have shown considerable promise in their explanatory power~\citep{roberts1991development,bacsar2004parameterized} and help address the issue of large item universes. But they suffer from a major issue: consideration sets are almost always unobserved (except in carefully controlled experimental settings) and cannot in general be identified from observed choice data~\citep{jagabathula2023demand,nierop10:_retriev_unobs_consid_sets_househ_panel_data}. However, if it were possible to derive meaningful information about consideration probabilities from rankings, this would be very valuable: for instance, in recommender systems, we can think of the core goal as identifying items with high utility but low consideration probability. These are the types of items that users would not discover naturally but still enjoy.

\paragraph{The present work.}
In this paper, we define the natural consider-then-rank model obtained by augmenting top-$k$ Plackett--Luce with the independent-consideration rule~\citep{manzini2014stochastic}, in which each item advances to the ranking stage randomly and independently with an unknown item-specific \emph{consideration probability}.
We term this model \emph{Plackett--Luce with consideration} (PL+C).
One might hope that richer observations---a ranking of $k$ items rather than just a single choice---would make it feasible to identify consideration probabilities, at least for large~$k$.
Unfortunately, our first result is negative: regardless of $k$, there are infinite families of consideration probabilities that generate the same distribution of observed data.
However, we show that it is possible to derive meaningful bounds on consideration probabilities, despite their non-identifiability.
In addition to observations of rankings, some (but not all) of our bounds draw on two types of information for learning about consideration probabilities: (1) known item utilities and (2) a lower bound on expected consideration set size.
(Later, we discuss settings where such information is indeed available.)
First, we derive relative bounds on the consideration probabilities of different items, of the form ``if item~$i$ has consideration probability $p_i$, then $j$'s consideration probability is greater than $f(p_i)$.''
The intuition is that if $i$ has higher utility then $j$, then we would expect to see $i$ ranked highly more often than $j$---but if instead $j$ outperforms $i$, then consideration must be the culprit, and $j$'s consideration probability must be higher than $i$'s.
Quantitatively, the $j$-vs.-$i$ gap in the frequency of being ranked highly tells us how much more often $j$ is considered than $i$. To make this bound more useful, we show how to infer that $i$ has higher utility than $j$ despite confounding from consideration, which invalidates standard Plackett--Luce inference.

These relative bounds provide useful information, but, absent additional anchor points, they still admit a large range of possible consideration probabilities. 
Motivated by this limitation, we then seek absolute upper and lower bounds on consideration probabilities.
For lower bounds, na\"{i}vely, we might hope that the fraction of observed rankings in which item $i$ appears would be a lower bound on $i$'s consideration probability.
However, this idea fails in our setting, as we assume that we only observe rankings of exactly $k$ items, and that any instance where fewer than $k$ items were considered is discarded before we observe it (or, equivalently, a consideration set is resampled).
However, we are able to rescue this approach if we can assume a mild bound on expected consideration set size, since we can then upper bound the probability that a sample was discarded.
Next, we turn to upper bounds on items' consideration probabilities. Our absolute upper bounds use exogenous knowledge of item utilities, which we learn in our data from survey questions with explicit consideration sets. 
%prove how to combine known item utilities and a lower bound on expected consideration set size to achieve an upper bound on .
Thus, if an item $i$ is ranked first less often than we would expect given its utility, then it must not be considered very often relative to the other items.
Using a pessimistic upper bound of 1 for other items' consideration probabilities can then yield a nontrivial upper bound on $i$'s consideration probability.
To wrap up our theoretical contributions, we provide algorithms to combine our absolute and relative bounds by propagating over a directed acyclic graph induced by our relative bounds.

Finally, we demonstrate how our methods can reveal meaningful information about consideration probabilities on real data.  In a psychology experiment about perceptions of U.S.\ history~\citep{putnam2018collective}, participants completed several tasks related to their view of the historical importance of particular U.S.\ states, such as naming the three states they believe contributed most to U.S.\ history. They were also provided with a random set of 10 states and asked to rate their percentage contribution to U.S.~history. (We convert the numerical scores to rankings.) These two settings, \emph{Top-3} and \emph{Random-10}, provide us both observations from fixed consideration sets (Random-10) and from unknown consideration sets (Top-3). We can thus estimate utilities in the absence of consideration from the Random-10 data and then estimate consideration from the Top-3 data, using our bounds. Our results on this data align with expectations about state history: original colonies Massachusetts and Virginia were the states with the highest range of possible consideration probabilities ($0.59$ to $1.00$), while Missouri, an anecdotally often-forgotten state~\citep{fotgottenstates}, had the lowest upper bound on consideration probability ($0.11$).

\section{RELATED WORK}

There are several approaches to adding consideration to choice models. The one closest to our work adds a consideration stage to random utility models~\citep{ben1995discrete,roberts1991development,nierop10:_retriev_unobs_consid_sets_househ_panel_data,bacsar2004parameterized}, following the formulation of \citet{manski1977structure}. Another line of work uses a different choice model, assuming choosers have a strict preference ordering over items and deterministically pick the highest ranked item they consider~\citep{manzini2014stochastic,cattaneo2020random}. An extension of this model allows the preference orderings to be stochastic~\citep{jagabathula2023demand}. There are also a variety of approaches for modeling the consideration stage, including feature-based constraints~\citep{ben1995discrete}, rational utility-maximization~\citep{roberts1991development}, and the independent consideration model we use~\citep{manzini2014stochastic}.

The issue of non-identifiability likewise has a number of proposed solutions. In the most straightforward approach,  experimenters directly ask item availability questions to choosers (e.g., ``is item $i$ available?'' or ``did you consider item~$i$?'')~\citep{roberts1991development,ben1995discrete,suh2009role}, which is especially useful when trying to estimate constraint-based consideration models. A similar approach in online shopping settings uses data about which items customers view to directly infer consideration sets~\citep{gu2012identifying,moe2006empirical}. Another strategy leans on knowing when a default ``outside option'' (i.e., making no selection) was selected and assuming that it is only chosen when no items are considered~\citep{manzini2014stochastic,jagabathula2023demand}. Yet another approach supposes we have observations of choices over time as item features vary and uses the idea that demand for highly considered items will respond more to changes in features~\citep{abaluck2021consumers}. Finally, several papers make assumptions about the parametric form of consideration probabilities as a function of chooser and item features and estimate those parameters from choice data~\citep{nierop10:_retriev_unobs_consid_sets_househ_panel_data,bacsar2004parameterized}, although this approach makes inference computationally challenging.

While consideration sets have been extensively explored in single-choice setting, they have received limited attention in the ranking literature; to our knowledge, our work is the first to recover information about consideration probabilities under Plackett--Luce. \citet{palma2017improving} discusses the idea of consideration sets in relation to rankings, but only uses them to identify a relevant prefix of rankings in survey data (items ranked above the outside option are treated as the implied consideration set). \citet{fok2012rank} mention the idea of applying consideration sets to rankings, but do not pursue this direction as they focus on survey settings where respondents are forced to rank all available items. We note that the Plackett--Luce model is sometimes referred to as the ``rank-ordered logit'' in the econometrics literature~\citep{beggs1981assessing,hausman1987specifying}, while ``Plackett--Luce'' is more common in computer science (for length-2 rankings, a.k.a.\ pairwise comparisons, the model is also called ``Bradley--Terry''~\citep{bradley1952rank}). Much of the existing computational work on the Plackett--Luce model has focused on utility inference~\citep{guiver2009bayesian,maystre2015fast,zhao2016learning,liu2019learning}.

\section{PRELIMINARIES}
In a ranking setting, we have a universe of items $\U = \{1, \dots, n\}$ and observe a collection of length-$k$ rankings of the form $r = \tup{r_1, \dots, r_k}$, where $k \le n$ is fixed and each $r_i \in \U$ is distinct.

The Plackett--Luce model~\citep{plackett1975analysis,luce1959individual} posits that each item~$i$ has a \emph{utility}~$u_i \in \mathbb{R}$ and that rankings are formed by a sequence of choices with the probability of selecting $i \in \mathcal{U}$ proportional to $\exp(u_i)$, first choosing a top-ranked item, then a second-ranked item (distinct from the first), etc. We assume these $k$ choices are made from only a subset of items, a \emph{consideration set} $C\subseteq \U$ where $|C| \ge k$.
The probability of observing ranking $r$ under a Plackett--Luce (PL) model with consideration set $C$ is
\begin{equation}\label{eq:pl}
  \PL(r \mid C) = \prod_{i = 1}^k \frac{\exp(u_{r_i})}{\sum_{j \in C \setminus \{r_1, \dots, r_{i-1}\}} \exp(u_{j})}
\end{equation}
if $\set{r_1, \ldots, r_k} \subseteq C$; otherwise, $\PL(r \mid C) = 0$.

As our model of consideration, we assume each item $i$ has a positive \emph{consideration probability} $p_i \in (0, 1]$ and that items are considered independently, conditioned on the constraint ${|C| \ge k}$. (We do not observe any cases where fewer than $k$ items are considered---either such instances are thrown out before we observe them, or the chooser considers a new set.)  Thus, the probability of considering a set $C$ with $|C| \ge k$ is
\begin{equation}\label{eq:consider}
  \PrC(C) = \frac{1}{z_{k, p}} \Big(\prod_{i \in C} p_i\Big) \Big(\prod_{j \in \mathcal{U}\setminus C} (1-p_j)\Big),
\end{equation}  
where $z_{k, p} = \sum_{C \subseteq \U, |C| \ge k} \left(\prod_{i \in C} p_i\right) \prod_{j \in \mathcal{U}\setminus C} (1-p_j)$ normalizes the probabilities given the condition $|C| \ge k$. For any set $C$ with $|C| < k$, we define $\PrC(C) = 0$.

Combining \Cref{eq:pl,eq:consider} yields our full model of rankings, \emph{Plackett--Luce with consideration} (PL+C), in which we sum over all possible consideration sets, weighted by their probabilities, and apply Plackett--Luce:
\begin{equation} \label{eq:PLC}
  \PLC(r) = \sum_{C \subseteq \U} \PrC(C) \cdot \PL(r \mid C).
\end{equation}
Finally, we introduce some additional notation. Let $\R$ be the set of length-$k$ rankings over $\U$ and let $\R_{i\le \ell}\subset \R$ be the set of rankings that contain $i$ in any of the top $\ell$ positions. For a set of rankings $R$, let $\PLC(R) = \sum_{r \in R} \PLC(r)$, and likewise for $\PL(R\mid C)$.

\section{(RELATIVE) CONSIDERATION PROBABILITY INFERENCE}

\subsection{Inferring exact values is impossible}

We begin with an impossibility result: even with complete information about $\PLC(r)$ for every ranking $r$---and even with complete knowledge of the utilities $u_1, \ldots, u_n$ for all items---we cannot infer the consideration probabilities $p_1, \ldots, p_n$ due to non-uniqueness. See \Cref{sec:deferred-proofs} for all omitted proofs.

\begin{restatable}{restateablethm}{Nonidentifiability}\label{thm:nonidentifiability}
For all $n \ge 1$ and $1 \le k  \le n$, consideration probabilities are not identifiable in the PL+C model. That is, there are multiple sets of consideration probabilities that generate the same distribution over rankings, with fixed utilities $u_i$ for each $i \in \U$.
\end{restatable}

\subsection{Inferring relationships among consideration probabilities is possible}

\Cref{thm:nonidentifiability} says that we cannot hope to compute exact consideration probabilities.
Nevertheless, we will show that, in certain circumstances, we can conclude something about the \emph{relative} values of two items' consideration probabilities. To do this, we will first need the following lemma about the probability of observing items in the top $\ell$ positions under Plackett--Luce.

 \begin{restatable}{restateablelemma}{PLTopL}\label{lemma:pl-top-l}
For two items $i, j \in \U$ with $u_i > u_j$, let $C \subseteq \U$ with $i \in C$ and $j \notin C$ and let $C_{i \rightarrow j} = C \setminus \set{i} \cup \set{j}$ be $C$ with $i$ replaced by $j$. Let $\ell \le k \le |C|$. Under Plackett--Luce, $i$ is more likely to be in the top $\ell$ positions of a length-$k$ ranking with consideration set $C$ than $j$ is with consideration set $C_{i \rightarrow j}$. That is, 
\begin{equation} \nonumber
   \PL(\R_{i \le \ell} \mid C) > \PL(\R_{j \le \ell} \mid C_{i \rightarrow j}). 
\end{equation}
\end{restatable}

Additionally, if both $i$ and $j$ are considered, it is easy to show that $i$ (the higher utility item) is more likely to be highly ranked than $j$. Thus, if both are equally likely to be considered, we would expect to see $i$ appearing more often in high rank positions than $j$. If, to the contrary, we observe that $j$ is chosen more frequently than $i$, then it must be the case that $j$ is considered more often.  That is: if we know the utilities of items, then we can use flips in top-$\ell$ ranking rates to identify which items are considered more frequently than others. We formalize this intuition in the following theorem.
\begin{theorem}\label{thm:consider-flip}
  Consider two items $i,j$ with $u_i > u_j$ in a PL+C model. If, for some $\ell$, $\PLC(\R_{i \le \ell}) \le \PLC(\R_{j \le \ell})$, then $p_i \le p_j$.
\end{theorem}
We defer the proof of \Cref{thm:consider-flip}, as this result is effectively subsumed by a more powerful theorem that also gives us information about the relative sizes of $p_i$ and $p_j$:

\begin{restatable}{restateablethm}{ConsiderGap}\label{thm:consider-gap}
   Consider two items $i,j$ with $u_i > u_j$ in a PL+C model. If $ c = \PLC(\R_{i \le \ell}) /  \PLC(\R_{j \le \ell}) \le 1$ for some $\ell$, then 
   \begin{equation}\label{eq:mult-gap}
     \textstyle \frac{p_i}{1-p_i} \le c \cdot \frac{p_j}{1-p_j}.
   \end{equation}
\end{restatable}
The proof isolates rankings that contain $i$ and not $j$ (and vice versa) and uses a bijection between those rankings and their associated consideration sets in conjunction with \Cref{lemma:pl-top-l} to establish the claim. Intuitively, the terms $p_i(1-p_j)$ and $p_j(1-p_i)$ come from considering $i$ and not $j$ or $j$ and not $i$, which we then cross-divide to group like variables. Taking $c=1$, \Cref{thm:consider-gap} implies \Cref{thm:consider-flip}, since $x/(1-x)$ is a monotonically increasing function of $x$ for $x \in [0, 1)$. Additionally, the smaller $c$ is (i.e., the more we see $j$ ranked higher than $i$, despite its lower utility), the larger the gap between $p_i$ and $p_j$ must be. We can rearrange \eqref{eq:mult-gap} to be either an upper bound on $p_i$ in terms of $p_j$
    or a lower bound on $p_j$ in terms of $p_i$:
%    \begin{subequations}
      \begin{align}\label{eq:upper-bound-on-pi}
      p_i &\le \textstyle\frac{c p_j}{1-p_j + c p_j} \\
      \label{eq:lower-bound-on-pj}
      p_j &\ge \textstyle\frac{p_i}{c - cp_i + p_i}.
    \end{align}
%    \end{subequations}

Our relative bounds appear to require exogenous knowledge of utilities, but we only need to know that $u_i > u_j$, not exact values. This can be inferred from rankings: if we only consider rankings in which both $i$ and $j$ appear (so we know both were considered), then whichever is ranked higher more often has higher utility, given enough observations. Formally:    
     \begin{restatable}{restateablelemma}{InferRelativeUtility}\label{lemma:infer-relative-utility}
 Let $R$ be a random top-$k$ ranking generated by Plackett--Luce with consideration. For items $i, j \in \U$, let $i \succ_R j$ denote that $i$ is ranked higher in $R$ than $j$. If $\Pr(i \succ_R j \mid i, j \in R) > 1/2$, then $u_i > u_j$.
\end{restatable}
We note that this approach to identifying relative utilities is prone to sampling noise for rare items.

\section{LOWER BOUNDS}
\label{sec:lower-bounds}

Suppose that we observe a set of rankings generated by a PL+C model, where the underlying consideration probabilities $p_1, \ldots, p_n$ are unknown.  Even if we have \emph{a priori} knowledge of the underlying utilities $u_1, \ldots, u_n$, \Cref{thm:nonidentifiability} says that we cannot compute the values of $p_1, \ldots, p_n$.  But we might hope that \Cref{thm:consider-flip} (or the more powerful \Cref{thm:consider-gap}) would allow us to leverage the empirical data into knowledge about their relative consideration probabilities: that theorem might allow us to propagate a lower bound on $p_i$ into a lower bound on $p_j$.  But for this kind of implication to be meaningful, we need a nontrivial starting point---that is, some way to infer that $p_i$ is bounded away from zero, say $p_i > \eps$.  

How might we get started?
While it is tempting to think that item $i$'s rate of occurrence in the observed length-$k$ rankings would lower bound $p_i$, the fact that we condition our samples on $|C| \ge k$ means that this relationship may not hold. For example, if each of $n = 5$ items has an identical consideration probability $p > 0$ and identical utility, then each of those five items occurs as the top choice with probability $0.2$, by symmetry---but that is true for \emph{any} value of $p > 0$, whether $p \ge 0.2$ or $p = 0.01$.  So $p_i \ge \PLC(\R_{i \le k})$ may not hold.
However, if we make some relatively mild assumptions about the \emph{expected} consideration set size, then we can use item occurrence rates and a Chernoff bound to get a lower bound on their consideration probability. From a practical perspective, we note that measuring mean consideration set size is a common task in consumer research~\citep{hauser1990evaluation}.

 \begin{restatable}{restateablelemma}{Chernoff}\label{lemma:chernoff}
   Suppose $\sum_{i \in \U} p_i \ge \alpha k$ for some $\alpha > 1$, and let $X$ be a random variable denoting the size of a random consideration set (prior to conditioning on there being at least $k$ considered items). Then
   \begin{equation}\nonumber
     \Pr(X \le k) \le \left(\alpha e^{1-\alpha}\right)^{ k}.
     \end{equation}
\end{restatable}

Given this bound on the cases where $|C| < k$, we can correct our earlier hopeful argument.  

\begin{restatable}{restateablethm}{ConsiderInitialLB}\label{thm:consideration-initial-LB}
  Suppose $\sum_{i \in \U} p_i \ge \alpha k$ for some $\alpha > 1$. For any ${i \in \U}$, the consideration probability for $i$ is lower bounded as
  \begin{equation} \nonumber
    p_i \ge \PLC(\R_{i \le k}) \cdot \left[1- \left(\alpha e^{1-\alpha}\right)^{ k}\right].
  \end{equation}
\end{restatable}

Algorithmically, we can use \Cref{thm:consideration-initial-LB} to get initial lower bounds on consideration probabilities, and then \Cref{thm:consider-gap}---in particular, in the lower bound form \eqref{eq:lower-bound-on-pj}---to propagate them to yield lower bounds on other items' consideration probabilities. 
Define a directed graph $G = \tup{V, E}$ with $V = \U$ and an edge $\tup{i, j}$ for each $i, j \in \U$ such that $u_i > u_j$ and $\PLC(\R_{i \le \ell}) < \PLC(\R_{j \le \ell})$ for some $\ell$. Each edge $\tup{i, j}$ represents a pair of items where \Cref{thm:consider-gap} gives the lower bound \eqref{eq:lower-bound-on-pj} for $p_j$ as a function of $p_i$. Because utilities strictly decrease along every edge, and thus along every path, the graph $G$ is acyclic. We can therefore topologically sort $G$ and use this ordering to propagate lower bounds on consideration probabilities. This procedure is formalized in \Cref{alg:lower-bounds}.

\begin{restatable}{restateablethm}{LBcorrectness}\label{thm:lower-bound-alg-correctness}
\Cref{alg:lower-bounds} returns lower bounds $b_i$ for each $i \in \U$ such that $p_i \ge b_i$ in time $O(kn^2)$.
\end{restatable}
In the proof, we show that ${\forall i \in \U : b_i \le p_i}$ is an invariant throughout the execution of the algorithm.  The initial values of $b_i$ set in Line~\ref{alg:initial-settings-of-b_i} satisfy the condition by \Cref{thm:consideration-initial-LB}; all updates to values of $b_i$ in Line~\ref{alg:updated-settings-of-b_i} maintain this property by  \eqref{eq:lower-bound-on-pj}, the lower-bound form of \Cref{thm:consider-gap}.  
%\end{proof}

\begin{algorithm}[tb]
   \caption{Consideration lower bounds.}
   \label{alg:lower-bounds}
\begin{algorithmic}[1]
\algsetup{
  indent=1.5em,
  linenosize=\scriptsize}
    \REQUIRE items $\U$, (estimates of) utilities $u_i$, (estimates of) top-$\ell$ ranking probabilities $\PLC(\R_{i \le \ell})$, lower bound $\alpha k$ on expected consideration set size ($\alpha > 1$)
    \STATE \label{alg:initial-settings-of-b_i} $b_i \gets \PLC(\R_{i \le k}) \cdot \left[1- \left(\alpha e^{1-\alpha}\right)^{ k}\right]$ for each $i \in \U$  
%    \STATE $E\gets \set{\vcenter{\hbox{$\displaystyle \tup{i, j} \subset \U \times \U \mid
%      \begin{array}[t]{@{}l@{}}u_i > u_j \text{ and } \\
%        \quad \exists \ell : \PLC(\R_{i \le \ell}) < \PLC(\R_{j \le \ell})
%      \end{array}
%      $}}}$
\STATE $E\gets \{\tup{i, j} \in \U \times \U \mid u_i > u_j$\\ \quad \qquad $\text{ and } \exists \ell : \PLC(\R_{i \le \ell}) < \PLC(\R_{j \le \ell})\}$
    \STATE $G \gets \tup{\U, E}$
    \FORALL {$i$ in a topological sort of $G$}
      \FORALL{out-neighbors $j$ of $i$ in $G$}
        \FOR{$\ell  = 1, \dots, k$}
        \STATE $c \gets \PLC(\R_{i \le \ell}) / \PLC(\R_{j \le \ell})$
          \IF {$c \le 1$}
            \STATE \label{alg:updated-settings-of-b_i} $b_j \gets \max\left\{b_j, \frac{b_i}{c - c b_i + b_i}\right\}$
          \ENDIF
        \ENDFOR
      \ENDFOR
    \ENDFOR
    \RETURN $b_i$ for each $i \in \U$ such that $p_i \ge b_i$
  \end{algorithmic}
\end{algorithm}

\section{UPPER BOUNDS}\label{sec:upper-bounds}

We now turn to \emph{upper} bounds on consideration probabilities, in the same spirit as for lower bounds in \Cref{sec:lower-bounds}.   We will again be able to use \Cref{thm:consider-gap} to tighten our initial upper bounds. However, upper bounding consideration probabilities turns out to be a bit trickier, so we will have to build some more technical infrastructure.

%\dln{Should we talk about Independence of Irrelevant Alternatives here, or in discussion? KT: I think in the discussion}

Our approach to inferring an upper bound on $p_i$ is based on the intuition that increasing the consideration probability of a different item $j$ only makes $i$ less likely to be chosen; increasing the competition for $i$ (by making $j$ more likely to be an additional competitor) can only hurt $i$'s chances. We can then use this idea to bound $i$'s winning chances against a hypothetical scenario where all other consideration probabilities are 1, in which case the consideration mechanism is easy to analyze. But, surprisingly, that intuition about increasing $j$'s consideration probability turns out to be false (but not by too much).  As a result, our lemma statements will need to be a bit more technical. (For example, suppose that $\U = \set{1, 2, 3}$ with $u_1 = u_2 = +\infty$ and $u_3 = -\infty$. Let $k = 2$, and let $p_1 = 1$ and $p_2 = 0.02$.  Item 3 never wins, but its presence affects the relative strengths of items 1 and 2. If $p_3 = 0$, then the only consideration set that we ever observe is $\set{1, 2}$, and so item 2 wins 50\% of the time.  But if $p_3$ is increased to $1$, then 98\% of the observed consideration sets are $\set{1, 3}$; the chance that item 2 wins drops from 50\% to 1\%.) It turns out that the issue that arises here is specifically about cases in which the consideration set is of size \emph{exactly} $k$. Fortunately, it is rare that a length-$k$ ranking emerges from a consideration set of size exactly equal to $k$ (under our ongoing mild assumptions about $\sum_i p_i$):
\begin{restatable}{restateablelemma}{ConsiderExactlyK}\label{lemma:consider-exactly-k-bound}
  If $\sum_{i \in \U} p_i \ge \alpha k$ for some $\alpha > 1$, then
  \begin{equation} \nonumber
    \sum_{C \subseteq \U, |C| = k} \PrC(C) \le {\textstyle \frac{\left(\alpha e^{1-\alpha}\right)^{ k}}{1 - \left(\alpha e^{1-\alpha}\right)^{ k}}}.
  \end{equation}
\end{restatable}

We use this bound on the probability that $|C| = k$ to bound the amount by which $i$'s chances of winning  increases when $j$'s consideration probability decreases. 

\begin{restatable}{restateablelemma}{IncreasePj}\label{lemma:increase-pj}
Suppose $\sum_{i \in \U} p_i \ge \alpha k$ for some $\alpha > 1$. Let $i, j \in \U$ with $i \neq j$. Let $\PLC'$ represent PL+C probabilities if we replace $p_j$ with $p_j' > p_j$. Doing so cannot increase the probability $i$ is ranked first, up to an additive error term: 
    \begin{align*} \nonumber
      \PLC'(\R_{i = 1}) &\le \PLC(\R_{i = 1}) + {\textstyle \frac{\left(\alpha e^{1-\alpha}\right)^{ k}}{1 - \left(\alpha e^{1-\alpha}\right)^{ k}}}.
%    \end{equation}
\intertext{If at least $k+1$ consideration probabilities are $1$ after increasing $p_j$, the inequality holds with no error:}
%    \begin{equation} \nonumber
      \PLC'(\R_{i = 1}) &\le \PLC(\R_{i = 1}).
    \end{align*}
\end{restatable}
The idea behind the proof is to group consideration sets containing $i$ into pairs where one contains $j$ and the other does not. Increasing $j$'s consideration probability makes it more likely we observe the consideration set with $j$, which then makes it less likely we see $i$ ranked first. However, this pairing fails if the first set in the pair has size $k$, since then its partner without $j$ is infeasible. \Cref{lemma:consider-exactly-k-bound} allows us to bound the error introduced by these failed pairings. On the other hand, if $k+1$ consideration probabilities are 1, then we never see the problematic size-$k$ consideration sets.

Finally, using \Cref{lemma:increase-pj}, we can upper bound $p_i$ by comparing how often $i$ is ranked first under the PL+C model to how often we would expect it to be ranked first if all other consideration probabilities were 1, which is easy to compute as a function of $p_i$.

\begin{restatable}{restateablethm}{UBRepeated}\label{thm:upper-bound:repeated-applications}
    Suppose that $\sum_{i\in\U} p_i \ge \alpha k$ for some $\alpha > 1.$ Then for any $i\in\U$,
     \begin{equation*}
       p_i \le \frac{\sum_{j \in \U}\exp(u_j)}{\exp(u_i)} \cdot \left(\PLC(\R_{i=1}) + \frac{k \left(\alpha e^{1-\alpha}\right)^{ k}}{1 - \left(\alpha e^{1-\alpha}\right)^{ k}}\right).
     \end{equation*}
\end{restatable}

Just as we did with lower bounds, we can use \Cref{thm:upper-bound:repeated-applications} to find initial upper bounds on every $p_i$ and then use \Cref{thm:consider-gap} to tighten these bounds, this time using the upper bound formulation \eqref{eq:upper-bound-on-pi}. The algorithm is identical in structure to the lower-bounding procedure; see \Cref{sect:supplement:algorithms} for pseudocode and a proof of correctness.
\iffalse
However, we need to propagate bounds in the opposite direction, so every edge in the graph $G$ is reversed (but it remains acyclic, since utilities now increase along every edge). We also take a minimum rather than maximum to find the tightest upper bounds.
\fi

%\include{improved-upper-bound}

\section{APPLICATION TO DATA}

%\begin{figure}

%\centering
%
%  \caption{Feasible consideration probability intervals for the 50 U.S.~states. }
%    \label{fig:bounds}
%\end{figure}

\begin{figure*}[t]
  \centering
  \resizebox{0.9\columnwidth}{!}{
    \begin{tikzpicture}[rotate=90,xscale=-.5,yscale=-1.5]
      \tikzset{
        state/.style = {draw, fill=gray!25, draw, inner sep=0.025pc, font=\tt\footnotesize, circle},
      }
      \foreach \state/\x/\y in {
        NY/-5/0, PA/-3/0, VA/-4/1, MA/-4/2,
        OR/6.75/0, 
        IL/-1/1, FL/0/1, TN/2/1, MO/3/1, CO/4.5/1, AZ/5.5/1, NM/7/1, WI/9/1, MN/10/1, MT/10.75/1, ND/11.5/1, SD/12.25/1, NE/13/1,
        NJ/-2/2, KY/0.5/2, MI/3/2, AR/4.875/2, KS/5.75/2, NV/6.5/2, IA/8/2, UT/9/2, ID/10/2,
        SC/-5/3, NC/-4/3, WV/2/3, RI/-0.5/3, AL/1/3, IN/3/3, OK/4/3, HI/5/3, AK/6/3, WY/10/3, TX/12/3,
        CT/-5/4, GA/-4/4, MS/-3/4, NH/-2/4, OH/-1/4, ME/0/4, VT/1/4, MD/3/4, LA/4/4, CA/12/4,
        WA/1/5, DE/5/5
      } {
        \node[state] (\state) at (\x,\y) {\state};
      }
      \begin{filecontents*}{data/reduced-dag-edge-list.csv}
AL,WA
AZ,UT
AZ,ID
AZ,HI
ID,AK
UT,AK
UT,WY
AR,HI
AR,WY
CA,DE
CO,AK
CO,AR
CT,WA
FL,RI
FL,NJ
ME,WA
NJ,OH
NJ,WV
NJ,ME
NJ,AL
OH,WA
RI,ME
VT,WA
WV,VT
GA,WA
IL,RI
IL,NJ
IA,AK
IA,HI
IA,WY
KS,AK
KS,HI
KS,WY
KY,RI
KY,AL
LA,WA
LA,DE
MD,WA
MD,DE
MI,WA
MI,HI
MI,WY
MI,OK
MI,IN
MN,UT
MN,ID
MN,IA
MS,WA
MO,AK
MO,AR
MO,RI
MO,AL
MO,MI
MT,ID
NE,ID
NE,WY
NV,AK
NV,HI
NV,WY
NH,WA
NM,UT
NM,ID
NM,IA
NM,NV
NY,VA
VA,MA
NC,NH
NC,OH
NC,GA
NC,MS
ND,ID
OR,CO
OR,UT
OR,ID
OR,IA
OR,NV
OR,KS
PA,VA
SC,NH
SC,OH
SC,CT
SC,GA
SC,MS
SD,ID
SD,WY
TN,AK
TN,AR
TN,OH
TN,WV
TN,MI
TN,KY
TX,CA
WI,UT
WI,ID
WI,IA
\end{filecontents*}
\csvreader{data/reduced-dag-edge-list.csv}{1=\fromstate,2=\tostate}{
        \draw (\fromstate) edge[->, shorten >=0.0625pc, shorten <=0.0625pc] (\tostate);
      }

      \iffalse\else          % JUST ADD THE \ELSE IF YOU WANT THIS
      %\iffalse\else 
      \node[state] (A) at (15,0) {\phantom{XX}};
      \node[state] (B) at (15,5) {\phantom{XX}};
      \draw (A) edge[->, shorten >=0.0625pc, shorten <=0.0625pc] (B);
      \node[below=0.35pc of A.west, anchor=north west, rectangle, outer sep=0pc, text width=8pc, font=\footnotesize, align=flush left] (A label) {higher utility\\lower consideration probability};
      \node[below=0.35pc of B.east, anchor=north east, rectangle, outer sep=0pc, text width=8pc, font=\footnotesize,  align=flush right] (B label) {lower utility\\higher consideration probability};
      \node[fit=(A)(B)(A label)(B label), draw] {};
      \fi
  \end{tikzpicture}
  }\qquad
%%%% !!! 
%%%% This figure relies on aamas-24-paper/data/states_with_bounds_sorted.csv
%%%% which is produced by
%%%% awk 'NR == 1; NR > 1 {print $0 | "sort -k 3 -r -n -t,"}' aamas-24-paper/data/states_with_bounds.csv
%%%% This dollar sign makes auctex happy: $
%%%%
    \resizebox{0.9\columnwidth}{!}{
    \begin{tikzpicture}[x=\columnwidth,y=2.6mm,yscale=-1]
      \tikzset{
        utility point/.style = {draw=red, fill=red!25, minimum size=0.01pc, inner sep=0.15pc, diamond},
        end point/.style = {fill=black, inner sep=0.125pc, circle},
        baseline end point/.style = {fill=blue!25, inner sep=0.2pc, rectangle},
        inner line/.style = {line width=0.125pc,black},
        outer line/.style = {line width=0.25pc,blue!25},
        state name/.style = {fill=white, inner xsep=0.0625pc, inner ysep=0pc, text=black, font=\small}
      }
      \draw (0,0) -- (0,50) ;
      \draw (1,0) -- (1,50) ;
      \draw[very thin, font=\small] (0,0) grid[xstep=0.2, ystep=50] (1,50);
      \foreach \i/\j in {0.0/0.0, 0.2/0.5, 0.4/1.0, 0.6/1.5, 0.8/2.0, 1.0/2.5} {
        \node[above=0pc of {\i,0}, gray, overlay] {\i};
        \node[below=0pc of {\i,50.5}, gray, overlay] {\j};
      };
      
      \begin{filecontents*}{data/states_with_bounds_utilities_sorted.csv}
,State,lower bound,lower bound baseline,upper bound,upper bound baseline,random-10 utility,shifted random-10 utility
45,Virginia,0.5893868466441953,0.48093228639321645,1.0,3.271622807514907,1.4489237,2.2955923
20,Massachusetts,0.5879790802055338,0.4985780033759892,1.0,3.5383679047345797,1.364786,2.2114546
31,New York,0.5138017592042636,0.5138017592042636,1.0,2.1913762599491213,1.5468887,2.3935573
37,Pennsylvania,0.4131173740672665,0.4131173740672665,1.0,2.209719167143043,1.3742012,2.2208698
46,Washington,0.3583250285602504,0.046709250836751245,1.0,2.0516093953928336,-0.3728415,0.47382703
7,Delaware,0.22597290913263723,0.15811946394366902,1.0,1.3977115021807691,0.44129437,1.2879629
4,California,0.18441504219250676,0.18441504219250676,0.8046292208733417,0.8046292208733417,0.7556648,1.6023333
42,Texas,0.10622029634727874,0.10622029634727874,0.5244232275583791,0.5244232275583791,0.88691497,1.7335835
44,Vermont,0.09478014868669438,0.01556975027891708,0.42901855243945014,0.42901855243945014,-0.13948703,0.7071815
19,Maryland,0.09445648502543029,0.09445648502543029,0.687052843944077,0.687052843944077,0.5519074,1.398576
47,West Virginia,0.08025116005258932,0.025949583798195133,0.3630136350488729,0.3630136350488729,-0.085803345,0.7608652
34,Ohio,0.07281035559738355,0.012801794673776266,0.3270690534059778,0.3270690534059778,-0.043437913,0.80323064
0,Alabama,0.06524817641174313,0.018337705884057896,0.3517130978297262,0.3517130978297262,-0.18207243,0.6645961
18,Maine,0.05756162826976182,0.016261739180202284,0.36425949323846835,0.36425949323846835,-0.2877775,0.55889106
16,Kentucky,0.057561628269761814,0.014185772476346673,0.3199599629853522,0.3199599629853522,-0.15810688,0.6885617
17,Louisiana,0.055359112102816285,0.055359112102816285,0.5342936185563998,0.5342936185563998,0.5137856,1.3604541
6,Connecticut,0.05432112875088848,0.05432112875088848,0.3919263468432895,0.3919263468432895,0.41437033,1.2610389
28,New Hampshire,0.05073233888412049,0.023873617094339526,0.3614332317512924,0.3614332317512924,0.13375673,0.98042524
38,Rhode Island,0.049747617641796325,0.020413672587913507,0.3293614265361677,0.3294666318549302,-0.2634146,0.583254
32,North Carolina,0.04029922496705092,0.02733356160076554,0.25250129955177864,0.25250129955177864,0.3355146,1.1821831
9,Georgia,0.04029922496705092,0.028371544952693346,0.2573302249471986,0.2573302249471986,0.31657082,1.1632394
35,Oklahoma,0.03372431317352716,0.0044979278583538234,0.33457761281675297,0.33457761281675297,-0.45076782,0.39590073
13,Indiana,0.03372431317352716,0.0048439223089964244,0.33304493273502467,0.33304493273502467,-0.44617635,0.4004922
39,South Carolina,0.033215467261689766,0.033215467261689766,0.21249577480578535,0.21249577480578535,0.3878113,1.2344799
29,New Jersey,0.03009230842608477,0.016607733630844883,0.21261102354279576,0.23484789618829308,-0.007165568,0.839503
49,Wyoming,0.025508297332083023,0.0020759667038556103,0.43009679235727233,0.43009679235727233,-0.8004208,0.04624772
3,Arkansas,0.025508297332083023,0.0031139500557834155,0.3326682853262951,0.3326682853262951,-0.5435564,0.30311215
10,Hawaii,0.025508297332083023,0.008649861266065043,0.3732228168852149,0.3732228168852149,-0.65858614,0.1880824
8,Florida,0.024219611544982127,0.024219611544982127,0.1776425726744549,0.207076964885266,0.029010529,0.8756791
21,Michigan,0.017151365599173674,0.005189916759639027,0.14067596694751472,0.22843050770851053,-0.27692991,0.5697386
1,Alaska,0.017151365599173674,0.006227900111566831,0.40382057852229486,0.40382057852229486,-0.84666854,0.0
5,Colorado,0.017151365599173674,0.005189916759639027,0.24943900565212637,0.27364210337606854,-0.45751894,0.3891496
23,Mississippi,0.016953728081487485,0.016953728081487485,0.2667877810946918,0.272822663236559,0.13791141,0.9845799
12,Illinois,0.01556975027891708,0.01556975027891708,0.21261102354279576,0.21317949894252128,0.08963796,0.9363065
41,Tennessee,0.008649861266065043,0.008649861266065043,0.051237054848241564,0.15622387970262375,-0.019701324,0.82696724
36,Oregon,0.005189916759639027,0.005189916759639027,0.21380437172138403,0.21380437172138403,-0.47337896,0.3732896
24,Missouri,0.0044979278583538234,0.0044979278583538234,0.05892467759938142,0.17586299323363894,-0.13811646,0.70855206
14,Iowa,0.004151933407711221,0.004151933407711221,0.165615182166497,0.29374670467849795,-0.65112907,0.19553947
2,Arizona,0.003459944506426018,0.003459944506426018,0.2611445304985936,0.2611445304985936,-0.6733915,0.17327702
15,Kansas,0.0024219611544982122,0.0024219611544982122,0.165615182166497,0.2689754779336067,-0.56303144,0.2836371
33,North Dakota,0.0020759667038556103,0.0020759667038556103,0.2529919713840486,0.3426692696861357,-0.805177,0.041491568
11,Idaho,0.0020759667038556103,0.0020759667038556103,0.2529919713840486,0.35558639995069774,-0.8421795,0.004489064
27,Nevada,0.001729972253213009,0.001729972253213009,0.165615182166497,0.2848657786869317,-0.62042934,0.2262392
40,South Dakota,0.001729972253213009,0.001729972253213009,0.28565491310783003,0.28565491310783003,-0.763102,0.08356655
22,Minnesota,0.001383977802570407,0.001383977802570407,0.2525611928363818,0.2525611928363818,-0.6399711,0.20669746
30,New Mexico,0.001383977802570407,0.001383977802570407,0.24093862543291614,0.24093862543291614,-0.59285975,0.2538088
48,Wisconsin,0.0010379833519278052,0.0010379833519278052,0.24637863977841304,0.24637863977841304,-0.61518705,0.23148149
25,Montana,0.0010379833519278052,0.0010379833519278052,0.3074503494572077,0.3074503494572077,-0.8366311,0.010037422
43,Utah,0.0006919889012852035,0.0006919889012852035,0.20099799182062753,0.31136801678091663,-0.70938677,0.13728178
26,Nebraska,0.0006919889012852035,0.0006919889012852035,0.28783277416478925,0.28783277416478925,-0.7706972,0.075971365
\end{filecontents*}
\csvreader{data/states_with_bounds_utilities_sorted.csv}{1=\index,2=\statename,3=\lowerbound,5=\upperbound,4=\lowerboundbaseline,6=\upperboundbaseline,7=\randomtenutility,8=\shiftedrandomtenutility}{
        \node[baseline end point] (baseline left) at (\lowerboundbaseline,\thecsvrow) {};
        \pgfmathparse{\upperboundbaseline > 1.0 ? 1 : \upperboundbaseline}
        \node[baseline end point] (baseline right) at (\pgfmathresult,\thecsvrow) {};
        \draw[outer line] (baseline left) edge (baseline right);
        \node[end point] (left) at (\lowerbound,\thecsvrow) {};
        \node[end point] (right) at (\upperbound,\thecsvrow) {};
        \pgfmathparse{\upperboundbaseline > 0.8 ? 1 : 0}
        \ifthenelse{\pgfmathresult>0}{
          \draw[inner line] (left) edge node[state name,pos=0.5] {\statename} (right);
        }{
          \draw[inner line] (left) edge (right);
          \node[right=0.125pc of baseline right,state name] {\statename};
        }
          \pgfmathparse{\shiftedrandomtenutility / 2.5}
          \node[utility point,opacity=0.75] (utility) at (\pgfmathresult,\thecsvrow) {};
      }

      \begin{pgfonlayer}{fg}
        \node[baseline end point] (outer L) at (0.65, 28) {};
        \node[baseline end point] (outer R) at (0.95, 28) {};
        \draw[outer line] (outer L) -- (outer R) node[below, outer sep=0.35pc, pos=0.5, state name, text width=6pc, align=flush center] (outer label) {Bounds from Theorems~\ref{thm:consideration-initial-LB} and \ref{thm:upper-bound:repeated-applications} (with $\alpha = 5$)\\{[top $0$-to-$1$ scale]}};
        \node[end point] (inner L) at (0.7, 36) {};
        \node[end point] (inner R) at (0.9, 36) {};
        \draw[inner line] (inner L) -- (inner R) node[below, outer sep=0.35pc, pos=0.5, state name, text width=6pc, align=flush center] (inner label) {Bounds from \Cref{alg:lower-bounds,alg:upper-bounds} (again with $\alpha = 5$) \\{[top $0$-to-$1$ scale]}};
        \node[utility point] (utility center) at (0.8, 45) {};
        \node[below=0.125pc of utility center, outer sep=0.35pc, state name, text width=6pc, align=flush center] (utility label) {Utility \\{[bottom scale]}};
      \end{pgfonlayer}
    \node[fit=(outer L)(outer R)(outer label)(inner L)(inner R)(inner label)(utility center)(utility label), draw, fill=white] {};
    \end{tikzpicture}
  }

    \caption{Left: the transitive reduction of the graph $G$ produced by \Cref{alg:lower-bounds} on the data from~\citep{putnam2018collective}; right: feasible consideration probability intervals for the 50 U.S.~states. On the left, an edge from state $i$ to state $j$ indicates that $i$'s inferred utility is larger than $j$'s but that $j$ appears in the top $\ell$ positions more often than $i$ (for some $\ell \in \set{1, 2, 3}$). Thus, we conclude from \Cref{thm:consider-flip} that $p_i \le p_j$. Nodes are labeled with states' postal abbreviations. On the right, the light blue intervals show the bounds from Theorems~\ref{thm:consideration-initial-LB} and \ref{thm:upper-bound:repeated-applications}, while the smaller black intervals show the bounds after tightening with \Cref{alg:lower-bounds,alg:upper-bounds}. Red diamonds show state utilities learned by PL on the Random-10 data, scaled additively so the minimum utility is 0.}
    \label{fig:swaps-bounds}
\end{figure*}

To illustrate how our theoretical machinery can be used to infer consideration behavior, we analyze an existing
dataset in which $\mathop{\approx} 2900$ American participants in a psychology experiment were asked to compare contributions of the 50 U.S.~states to U.S.~history~\citep{putnam2018collective}.  The data\footnote{Made available by~\citet{putnam2018collective} at \url{https://osf.io/tnjqs/} under a CC-BY 4.0 license.} include their responses to numerical questions (the \emph{Random-10 question:} estimate the percentage of state $X$'s contribution to U.S.~history, for each state $X$ in a set of ten randomly selected states) and ranking questions (the \emph{Top-3 question:} identify, in order, the three states that contributed the most to U.S.~history). For Random-10, we convert the numerical ratings into rankings by sorting.  Since the set of states being considered is fixed by the Random-10 question itself, we model these rankings using Plackett--Luce without consideration (PL). However, the Top-3 question forces participants to construct their responses without a reference list of candidates, instead having to recall the names of the states that they will list~\citep{brown1976recall}; indeed, participants are unlikely to have considered all 50 states, making the Top-3 setting well suited to PL+C.
 
According to \Cref{thm:nonidentifiability}, we cannot hope to learn consideration probabilities for the 50 U.S.~states. However, \Cref{alg:lower-bounds} (and its upper-bound analog; see \Cref{alg:upper-bounds} in \Cref{sect:supplement:algorithms}) enables us to at least infer bounds on these consideration probabilities given estimates of top-$\ell$ appearance rates $\PLC(\R_{i \le \ell})$, utilities $u_i$, and a lower bound on expected consideration set size $\alpha k$. Because the Top-3 and Random-10 data feature rankings over the same universe of items (the states), if we model these responses with PL+C and PL, respectively, then it is natural to assume that the underlying utilities are the same between the two questions. Working from this assumption, we fit\footnote{We used Rprop~\citep{riedmiller1993direct} on the full dataset with initial learning rate 0.05 and $L_2$ regularization strength $10^{-6}$, stopping when the squared gradient magnitude fell below $10^{-8}$. Since the PL model has a convex negative log likelihood, any well-tuned or robust optimizer would be expected to converge to a global optimum. While we used a train-test split for initial exploration, the results presented here are from the model fit to the entire dataset, as we do not have meaningful test metrics to evaluate. Given our assumptions, our bounds provably hold. Training takes less than a second on a 2018 MacBook Pro. Our code is available at \url{https://github.com/aoki-sherwoodb/bounding-consideration-probs}.} a PL model (implemented in PyTorch~\citep{paszke2019pytorch}) to the Random-10 data, learning the utility $u_i$ for each state and taking these to be the utilities under the full PL+C model for Top-3. For each state $i$, we then calculate the empirical estimates of $\PLC(\R_{i \le \ell})$ for $\ell = 1, 2, 3$: the proportion of the Top-3 rankings in which state $i$ appears in the first $\ell$ positions. Finally, we take $\alpha = 5$, assuming that participants on average consider at least $15$ states when ranking their top $3$.

Recall that \Cref{alg:lower-bounds} (and, similarly, \Cref{alg:upper-bounds}) relies on pairs of states whose utilities and top-$\ell$ ranking rates are flipped.\footnote{To identify these flips, we use utilities learned from Random-10 data rather than \Cref{lemma:infer-relative-utility}, avoiding high sampling error on states that rarely appear in the Top-3.} There are many such flips in this data, highlighting the importance of consideration in the Top-3 question. We visualize these flips in \Cref{fig:swaps-bounds} (left), which shows the directed acyclic graph $G$ constructed in \Cref{alg:lower-bounds}. (For clarity, we display the \emph{transitive reduction}~\citep{aho1972transitive} of $G$.) For each edge $\tup{i, j}$ in this graph, we know that state $i$ has a lower consideration probability than state $j$, despite having higher utility. For instance, we find that Massachusetts has higher consideration probability than Virginia, which in turn has higher consideration probability than New York and Pennsylvania.

We also compute upper bounds on consideration probabilities using \Cref{thm:upper-bound:repeated-applications}, with the same utilities $u_i$, values of $\alpha$ and $k$, and top-1 ranking probabilities $\PLC(\R_{i = 1})$ (again, using empirical estimates). Combining our lower and upper bounds yields feasible intervals on consideration probabilities for each state, which we display in \Cref{fig:swaps-bounds} (right). Our upper bounds reveal that, if our assumptions are valid, most states are considered less than 30--40\% of the time. Additionally, the bounds on consideration probabilities align with theories about why certain states were highly rated in the data~\citep{putnam2018collective}. Of the eight states with the highest lower bound, five are states drawn from the thirteen original colonies and commonly associated with the American Revolution (Massachusetts, Virginia, New York, Pennsylvania, and Delaware), two are the largest U.S.\ states by population (California and Texas), and the last, Washington, was hypothesized by \citet{putnam2018collective} to be confused by participants with the U.S.\ capital city Washington, D.C.\ and may also be easy to recall in the context of historical judgements due to its  namesake. 

%Our code is available on \href{https://github.com/aoki-sherwoodb/bounding-consideration-probs}{Github}. 

\section{DISCUSSION}
We formalized a natural model of ranking with consideration, adding independent consideration to Plackett--Luce. Despite showing that consideration probabilities are not identified in general, we derived relative and absolute bounds that allow us to learn about possible ranges of consideration probabilities from observed ranking data. Our data application demonstrates how these bounds can be used in practice to gain insight into consideration behavior. In addition to providing behavioral insights, approximately recovering consideration probabilities has exciting possible applications in recommender systems, which could be tailored to suggest items that have high utility but low consideration probability.

The astute reader will notice that our bounds use exact top-$\ell$ ranking probabilities and utilities, but that empirical estimates from observed rankings are necessarily noisy (and may reverse whether $u_i > u_j$ or $u_i < u_j$). Luckily, there are some easy fixes. For \Cref{thm:consider-gap}, we can use upper and lower confidence intervals on top-$\ell$ ranking rates to compute a range of possible values of $c$ (and, in particular, a high-probability guarantee under sampling error). To handle uncertainty in utility estimates, we can use confidence intervals to find only statistically significant utility-consideration flips.

%\subsection{Future work}

There remains much to explore regarding the PL+C model, and ranking with consideration sets more generally. While PL+C is clearly at least as powerful as Plackett--Luce, how much expressive power is gained by adding consideration? What kinds of distributions over rankings can and cannot be expressed by a PL+C model? Another natural extension of the Plackett--Luce model allows violations of the independence of irrelevant alternatives (IIA) assumption built into Plackett--Luce, such as the contextual repeated selection (CRS) model~\citep{seshadri2020learning}. How does the expressive power of CRS compare to PL+C? Does PL+C allow (apparent) IIA violations due to the consideration stage? (We expect so; the three-item example in \Cref{sec:upper-bounds} is reminiscent of non-IIA behavior.)

Another interesting question concerns computing PL+C probabilities efficiently. If we know utilities and consideration probabilities, the direct approach to computing $\PLC(r)$ using \Cref{eq:PLC} involves a sum over exponentially many consideration sets. Is it possible to exactly compute PL+C probabilities in polynomial time, or is it provably hard? We have derived two efficient approximation algorithms, but not a hardness proof or exact efficient algorithm. The first algorithm samples sufficiently many consideration sets and computes an empirical estimate of $\PLC(r)$, while the second groups together consideration sets with similar total utilities (with ``similarity'' defined by a tolerance $\epsilon$), computes their total consideration probabilities, and approximates their total utilities. Details of the algorithms are in \Cref{app:approx-alg}; we leave further exploration of this problem to future work.
 
\begin{restatable}{restateablethm}{ProbApproxalg}\label{thm:prob_approx-alg}
  There is a randomized $\epsilon$-additive approximation algorithm for $\PLC(r)$ with runtime $O(kn\log(1/\delta)/(z_{k,p}\epsilon^2))$, where both the approximation and runtime bounds hold with probability at least $1-\delta$.
\end{restatable}

 \begin{restatable}{restateablethm}{Approxalg}\label{thm:approx-alg}
   There is a deterministic algorithm approximating $\PLC(r)$ with multiplicative error at most $1+\epsilon$ and runtime $O(kn^2\max\{\log n, m\}/\epsilon)$, where $m = \max_{i \in \overline r}u_i$.
\end{restatable}
% Note: i removed the discussion of bound tightening using these algorithms, since on further reflection we would need the probability an item is *ranked first*, not just the probability of a single ranking....

The PL+C model also admits several natural extensions, such as incorporating different models of the consideration stage or allowing rankings of varied length. The latter direction may even simplify some of the consideration probability bounds, as some of the difficulty in our top-$k$ setting arises because of the conditioning that $|C| \ge k$ (but this conditioning is often needed given the ubiquity of top-$k$ data).

\subsection*{Acknowledgments}

Portions of this work were carried out while all authors were at Carleton College and while K.T.~was at Cornell University. Thanks to Aadi Akyianu, David Chu, Katrina Li, Adam Putnam, Sophie Quinn, Morgan Ross, Laura Soter, and Jeremy Yamashiro for helpful discussions. Comments are welcome.

%\subsection{Limitations}
%
%
%Finally, we highlight some limitations of our approaches. First, our absolute bounds are all contingent on a lower bound on mean consideration set size, which in some cases may be hard to obtain or test. (For instance, we make no claim that the value $\alpha = 5$ picked for demonstration in our data application is accurate.) However, measuring mean consideration set size is a common task in consumer research~\citep{hauser1990evaluation}, so realistic bounds can be available in practice. Further, our bounds all assume that the true model from which the data are generated is a PL+C model and that there is zero sampling error, so that we know exact utilities and PL+C probabilities. In practice, we only have estimates of these utilities and empirical top-$\ell$ probabilities, and the data may deviate from what is expressible under PL+C. Nonetheless, our methods provide a practical starting point for further investigation into ranking with consideration. 

% \subsubsection*{References}

%References follow the acknowledgements.  Use an unnumbered third level
%heading for the references section.  Please use the same font
%size for references as for the body of the paper---remember that
%references do not count against your page length total.

\bibliographystyle{apalike} 
\bibliography{references}

\clearpage
\appendix
\onecolumn

\include{supplement}
\end{document}

%% file: supplement.tex
\section{DEFFERED PROOFS}
\label{sec:deferred-proofs}

\Nonidentifiability*
\begin{proof}
 We begin with trivial cases. If $n = 1$, then the only observed ranking is $\tup{1}$, regardless of $p_1$. If $k = n$, then the only consideration set with nonzero probability is $\U$, so the consideration probabilities do not affect observed rankings. In both cases, arbitrary nonzero consideration probabilities produce the same distribution over rankings, so the probabilities are not identifiable.

 Now suppose $n > 1$ and $k < n$. We will define $k-1$ symmetric ``excellent'' items, $n-k$ symmetric ``good'' items, and one ``bad'' item:
 \begin{itemize}
 \item $k-1$ excellent items: $u_1, \dots, u_{k-1} = + \infty$ and $p_1, \dots, p_{k-1} = 1$,
 \item $n-k$ good items: $u_{k}, \dots, u_{n-1} = 1$ and $p_{k}, \dots, p_{n-1} = g$, and
 \item the one bad item: $u_{n} = - \infty$ and $p_n = b$,
 \end{itemize}
 for some fixed probabilities $g$ and $b$.  Excellent items are always considered and occupy the top $k-1$ positions in the ranking, so only two types of rankings have nonzero PL+C probability: ``good'' rankings with any of the $n-k$ good items in last (i.e., $k$th) place, and ``bad'' rankings with the bad item in last place.  We can only observe a bad ranking if the consideration set includes the bad item but none of the good items; we observe a good ranking if at least one good item was considered. (If neither the bad item nor any good item was considered, which happens with probability $(1-b)(1-g)^{n - k}$, the ranking is unobserved.)  Thus, after normalization, the probability of observing a bad ranking is
 \begin{equation} \label{eqn:nonidentifiability:def-of-c} \textstyle
   c = \frac{b \cdot (1-g)^{n - k}}{1 - (1-b)(1-g)^{n - k}}.
 \end{equation}
 By symmetry, all bad rankings have equal probability, as do all good rankings; thus the probability mass of $c$ is evenly divided over bad rankings, and the remaining $(1 - c)$ mass is evenly divided among good rankings. The distribution of rankings is thus governed by the single parameter $c$.
 To get multiple consideration probabilities yielding the same distribution, we can fix an appropriate value of $c$, select any $g \in (0, 1)$, and solve for $b$ in \eqref{eqn:nonidentifiability:def-of-c}. To complete the proof, we verify this equation is feasible (i.e., yields $b \in (0, 1]$) for sufficiently small values of $c$.  Write $\lambda = (1-g)^{n - k}$. Then \eqref{eqn:nonidentifiability:def-of-c} says that $c = b\lambda/(1 - (1-b)\lambda)$ or, rearranging, that $b = \smash{\frac{c}{1-c} \cdot \frac{1-\lambda}{\lambda}}$. Because $g \in (0, 1)$ implies $\lambda \in (0,1)$, this formula yields a value of $b \in (0,1]$ so long as $0 < c \le \lambda$.
\end{proof}

\PLTopL*
\begin{proof}
  We equivalently show  that the probability $i$ is \emph{not} in the top $\ell$ positions under $C$ is \emph{smaller} than the probability that $j$ is \emph{not} in the top $\ell$ positions under $C_{i \rightarrow j}$. We can find the probability that $i$ is not in the top $\ell$ positions by summing over every permutation of length-$\ell$ subsets of $C\setminus\{i\}$ and, for each permutation, multiplying the probabilities yielding that permutation under Plackett--Luce. When we replace $i$ with $j$, every Plackett--Luce probability strictly increases, since we replace the term $\exp(u_i)$ in every denominator with $\exp(u_j)$, which is strictly smaller since $u_i > u_j$. 
\end{proof}

\ConsiderGap*
\begin{proof}
  Since we are considering a fixed $\ell$ in this proof, write $\R_i = \R_{i \le \ell}$ and $\R_j = \R_{j \le \ell}$ for brevity. Additionally, let $\R_{i \setminus j} = {\R_{i} \setminus \R_{j}}$ and $\R_{i \cap j} = {\R_i \cap \R_j}$. We can partition $\R_{i}$ into the rankings that contain $i$ but not $j$ in the first $\ell$ positions, $\R_{i \setminus j }$, and the rankings that contain both in the first $\ell$ positions, $\R_{i \cap j}$. Analogously define $\R_{j \setminus i}$. Thus 
  $\PLC(\R_{i}) = \PLC(\R_{i \setminus j }) + \PLC(\R_{i\cap j})$ and
  $\PLC(\R_{j}) = \PLC(\R_{j \setminus i }) + \PLC(\R_{i\cap j})$.

  Suppose $c = \PLC(\R_{i})/  \PLC(\R_{j}) \le 1$. By the above equalities, then, we have
  \begin{align}
         &\PLC(\R_{i \setminus j}) + \PLC(\R_{i \cap j}) = c \cdot \PLC(\R_{j\setminus i}) + c \cdot \PLC(\R_{i \cap j})\notag \\
         &\Leftrightarrow  \PLC(\R_{i \setminus j}) + (1-c)\PLC(\R_{i \cap j}) = c \cdot \PLC(\R_{j\setminus i})\notag \\
         &\Rightarrow  \PLC(\R_{i \setminus j}) \le c \cdot \PLC(\R_{j\setminus i}) \label{ineq:ij-vs-ji}. 
      \end{align}
      We'll write out both sides of \eqref{ineq:ij-vs-ji} as a sum over consideration sets, starting with $\PLC(\R_{i \setminus j})$. Because $\PL(r \mid C) = 0$ if the items in $r$ are not all in $C$, for $\PLC(\R_{i \setminus j})$ we only need to sum over consideration sets that include $i$, which either also include $j$ or do not. Denote these two collections of consideration sets $\mathcal C_{i\cap j} = \{{C \subseteq \U} \mid {i \in C}, {j \in C}\}$ and $\mathcal C_{i\setminus j} = \{{C \subseteq \U} \mid {i \in C}, {j \notin C}\}$. Then
\begin{align}
  &\PLC(\R_{i \setminus j}) = \sum_{\substack{r \in \R_{i \setminus j}\\C \subseteq \U}}  \PrC(C) \cdot \PL(r \mid C) \nonumber\\
 &=\sum_{\substack{r \in \R_{i \setminus j}\\C \subseteq \mathcal C_{i \cap j}}} \PrC(C) \cdot \PL(r \mid C)  + \sum_{\substack{r \in \R_{i \setminus j}\\C \subseteq \mathcal C_{i \setminus j}}} \PrC(C) \cdot \PL(r \mid C). 
\label{eq:expansion-of-PLC-Ri-j}
\end{align}
Now we'll do the same for $\PLC(\R_{j\setminus i})$. We can construct a bijection $\phi: \R_{i \setminus j} \rightarrow \R_{j \setminus i}$ by taking a ranking $r \in \R_{i \setminus j}$ and replacing $i$ with $j$. We can also construct a bijection $\psi : \mathcal C_{i \setminus j} \rightarrow \mathcal C_{j \setminus i}$ by taking a set $C \in \mathcal C_{i \setminus j}$ and replacing $i$ with $j$. Using these bijections, we can decompose $ \PLC(\R_{j\setminus i})$ into sums over the same sets as in \eqref{eq:expansion-of-PLC-Ri-j}. For clarity, we'll write $r_{i \rightarrow j} = \phi(r)$ and $C_{i \rightarrow j} = \psi(C)$. We then have:
\begin{align}
  \PLC(\R_{j \setminus i}) =& \sum_{\substack{r \in \R_{j \setminus i}\\C \subseteq \U}} \PrC(C) \cdot \PL(r \mid C) \nonumber\\
  ={}& \sum_{\substack{r \in \R_{i \setminus j}\\C \subseteq \U}} \PrC(C) \cdot \PL(r_{i \rightarrow j} \mid C) \nonumber\\
  ={}&\sum_{\substack{r \in \R_{i \setminus j}\\C \subseteq \mathcal C_{i \cap j}}} \PrC(C) \cdot \PL(r_{i \rightarrow j} \mid C) \quad+\quad \sum_{\substack{r \in \R_{i \setminus j}\\C \subseteq \mathcal C_{i \setminus j}}} \PrC(C_{i \rightarrow j}) \cdot \PL(r_{i \rightarrow j} \mid C_{i \rightarrow j}). \label{eq:expansion-of-PLC-Rj-i}
\end{align}
      Combining \eqref{ineq:ij-vs-ji} with \eqref{eq:expansion-of-PLC-Ri-j} and \eqref{eq:expansion-of-PLC-Rj-i}, we then have
\begin{align}
        &\sum_{\substack{r \in \R_{i \setminus j}\\C \subseteq \mathcal C_{i \cap j}}}  \PrC(C) \cdot \PL(r \mid C) + \sum_{\substack{r \in \R_{i \setminus j}\\C \subseteq \mathcal C_{i \setminus j}}} \PrC(C) \cdot \PL(r \mid C)\notag \\
\le{}& c \cdot \sum_{\substack{r \in \R_{i \setminus j}\\C \subseteq \mathcal C_{i \cap j}}} \PrC(C) \cdot \PL(r_{i \rightarrow j} \mid C) \quad+\quad c \cdot \sum_{\substack{r \in \R_{i \setminus j}\\C \subseteq \mathcal C_{i \setminus j}}} \PrC(C_{i \rightarrow j}) \cdot \PL(r_{i \rightarrow j} \mid C_{i \rightarrow j}). \label{ineq:ij-vs-ji-expanded}
      \end{align}

  Now, we show the first sum on the left-hand side of \eqref{ineq:ij-vs-ji-expanded} is larger than the first sum on the right-hand side (i.e., comparing the summations over $\mathcal C_{i \cap j}$, for the moment disregarding the factor of $c$). To do so, consider some fixed $r^* \in \R_{i \setminus j}$. Recall that $u_i > u_j$. We will show that, for all $C \subseteq \mathcal C_{i \cap j}$, we have ${\PL(r^* \mid C)} > {\PL(r^*_{i \rightarrow j} \mid C)}$. Suppose $i$ appears at position $h$ in $r^*$. We then have 
    \begin{align*}
      & \PL(r^* \mid C)\\
      &=\prod_{s = 1}^{h-1} \frac{\exp(u_{r_s^*})}{\sum_{t \in C \setminus \{r_1^*, \dots, r_{s-1}^*\}} \exp(u_{t})} 
      \cdot \frac{\exp(u_i)}{\sum_{t \in C \setminus \{r_1^*, \dots, r_{h-1}^*\}} \exp(u_{t})}
      \cdot \prod_{s = h+1}^{k} \frac{\exp(u_{r_s^*})}{\sum_{t \in C \setminus \{r_1^*, \dots, r_{s-1}^*\}} \exp(u_{t})}. 
    \end{align*}
     $\PL(r^*_{i \rightarrow j} \mid C)$ is identical, except the middle probability has numerator $\exp(u_j)$ instead of $\exp(u_i)$ and the denominators in the bottom product include $\exp(u_i)$ instead of $\exp(u_j)$. Since $u_i > u_j$, this means $\PL(r^* \mid C) > \PL(r^*_{i \rightarrow j} \mid C)$. Summing over all $r^*$ and $C$ then yields
     \begin{align}
       \sum_{\substack{r \in \R_{i \setminus j}\\C \subseteq \mathcal C_{i \cap j}}} \PrC(C) \cdot \PL(r \mid C) > \sum_{\substack{r \in \R_{i \setminus j}\\C \subseteq \mathcal C_{i \cap j}}}  \PrC(C) \cdot \PL(r_{i \rightarrow j} \mid C).\label{ineq:first-sum}
     \end{align}
     Inequality \eqref{ineq:first-sum} certainly still holds after scaling the right side by $c \le 1$. We can therefore remove these sums from \eqref{ineq:ij-vs-ji-expanded} and maintain the inequality, since we are shrinking the left by more than the right. This leaves us with
    \begin{align}
      \sum_{\substack{r \in \R_{i \setminus j}\\C \subseteq \mathcal C_{i \setminus j}}} \PrC(C) \cdot \PL(r \mid C) 
      \le c \cdot \sum_{\substack{r \in \R_{i \setminus j}\\C \subseteq \mathcal C_{i \setminus j}}} \PrC(C_{i \rightarrow j}) \cdot \PL(r_{i \rightarrow j} \mid C_{i \rightarrow j}).\label{ineq:ij-vs-ji-right-sum}
    \end{align}
  Next, notice that to transform $\PrC(C)$ into $\PrC(C_{i \rightarrow j})$, we don't consider $i$ and instead consider $j$:
     \begin{align} \label{eq:PrC(psiC)-vs-PrC(C)}
       \PrC(C_{i \rightarrow j}) &= \textstyle \PrC(C) \cdot \frac{p_j(1-p_i)}{p_i(1-p_j)}. 
     \end{align}
    Applying \eqref{eq:PrC(psiC)-vs-PrC(C)} to \eqref{ineq:ij-vs-ji-right-sum} and rearranging the sums yields
  \begin{align}
      &\sum_{C \subseteq \mathcal C_{i \setminus j}} \Big[\PrC(C)  \sum_{r \in \R_{i \setminus j}} \PL(r \mid C) \Big] \notag \\
      \le{}& c\cdot \sum_{C \subseteq \mathcal C_{i \setminus j}} \Big[ \PrC(C_{i \rightarrow j}) \sum_{r \in \R_{i \setminus j}}\PL(r_{i \rightarrow j} \mid C_{i \rightarrow j}) \Big] \notag \\
      ={}& c\cdot {\textstyle \frac{p_j(1-p_i)}{p_i(1-p_j)}} \cdot \sum_{C \subseteq \mathcal C_{i \setminus j}} \Big[\PrC(C) \sum_{r \in \R_{i \setminus j}}\PL(r_{i \rightarrow j} \mid C_{i \rightarrow j})\Big]. \label{ineq:extract-ps}
    \end{align}
     Now, notice that $\sum_{r \in \R_{i \setminus j}}  \PL(r \mid C)$ is the probability that $i$ is ranked in the top $\ell$ positions under Plackett--Luce with consideration set $C$ (since $j\notin C$). Similarly, $\sum_{r \in \R_{i \setminus j}} \PL(r_{i \rightarrow j} \mid C_{i \rightarrow j})$ is the probability that $j$ is ranked in the top $\ell$ positions with consideration set $C_{i \rightarrow j}$, which is just $C$ with $i$ replaced by $j$. By \Cref{lemma:pl-top-l}, we therefore have $\sum_{r \in \R_{i \setminus j}}  \PL(r \mid C) > \sum_{r \in \R_{i \setminus j}} \PL(r_{i \rightarrow j} \mid C_{i \rightarrow j})$.
     
  We can therefore increase the right hand side of \eqref{ineq:extract-ps} to find
  \begin{align*}
    &\sum_{C \subseteq \mathcal C_{i \setminus j}} \Big[ \PrC(C)  \sum_{r \in \R_{i \setminus j}} \PL(r \mid C) \Big] \\
    \le{}&  c\cdot {\textstyle \frac{p_j(1-p_i)}{p_i(1-p_j)}} \cdot \sum_{C \subseteq \mathcal C_{i \setminus j}} \Big[ \PrC(C) \sum_{r \in \R_{i \setminus j}} \PL(r \mid C) \Big] .
  \end{align*}
    Dividing both sides by the left-hand side then yields $1 \le c \cdot \smash{\frac{p_j(1-p_i)}{p_i(1-p_j)}}$, or, rearranging, the desired relationship \eqref{eq:mult-gap}:
    \begin{align*}
      \textstyle
      \frac{p_i}{1-p_i} \le c \cdot \frac{p_j}{1-p_j}. \tag*{\qedhere}
    \end{align*} 
    \end{proof}
    
  \InferRelativeUtility*
  \begin{proof}
    Consider the set of top-$k$ rankings containing both $i$ and $j$, denoted $\mathcal R_{i \cap j}$. This set of rankings can be partitioned into two equal-size subsets, one in which $i$ is ranked above $j$ and the other in which $j$ is ranked above $i$. Moreover, there is a bijection between these sets of rankings defined by swapping the position of $i$ and $j$. Now suppose (towards proving the contrapositive of the claim) that $u_j \ge u_i$. In each pair of rankings defined by the above bijection, the one with $j$ ranked higher has higher probability under Placett--Luce. To see this, note that in the Plackett--Luce probabilities (\Cref{eq:pl}) of two corresponding rankings, swapping $i$ and $j$ only affects the ranking probability in the denominator (since the product of numerators commutes). The ranking with $j$ ranked higher has strictly smaller denominators in its Plackett--Luce probability for the choices between $j$ and $i$'s position (including $i$'s position, but not $j$'s), since the higher-utility item is removed from the set of options after it is chosen. Thus it is more likely under Plackett--Luce (or at least as likely in the case that $u_i =u_j$). We then have that in the bijection between the parts of $\mathcal R_{i \cap j}$ with $j$ and $i$ ranked higher, the set of rankings with $j$ ranked higher are more probable under Plackett--Luce \emph{with consideration} (for every feasible consideration set, the ranking with $j$ ranked higher is more probable). Thus, $\Pr(j \succ_R i \mid i, j \in R) \ge 1/2$. Taking the contrapositive, we find that if $\Pr(j \succ_R i \mid i, j \in R) < 1/2$ (or equivalently, $\Pr(i \succ_R j \mid i, j \in R) > 1/2$), then $u_i > u_j$.  
  \end{proof}

\Chernoff*

\begin{proof}
By hypothesis, $\E[X] = \sum_{i \in \U} p_i \ge \alpha k$. Thus,
    \begin{align*}
    \Pr(X \le k) &= \Pr(X \le (1/\alpha) \alpha k)\\
    &\le \Pr\left(X \le (1/\alpha) E[X]\right)\tag{since $E[X] \ge \alpha k$}\\
    &=  \Pr\left(X \le (1 - (1 - 1/\alpha)) E[X]\right)\\
    &\le \left(\frac{e^{-(1-1/\alpha)}}{(1/\alpha)^{1/\alpha}}\right)^{E[X]} \tag{Chernoff bound~\citep{mitzenmacher2017probability}}\\
    &\le \left(\frac{e^{-(1-1/\alpha)}}{(1/\alpha)^{1/\alpha}}\right)^{\alpha k}  \tag{$E[X] \ge \alpha k$ and the ratio is $\mathop{\le}1$}\\
    &= \left(\frac{e^{-\alpha+1}}{1/\alpha}\right)^{ k}\\
    &= \left(\alpha e^{1-\alpha}\right)^{ k}. \tag*{\qedhere}
  \end{align*}
\end{proof}

\ConsiderInitialLB*

\begin{proof}
  Consider a particular instance where a chooser formed a consideration set. Let $X = \sum_{i \in \U} X_i$ be the random variable representing the size of the consideration set, summing over indicator random variables (with $X_i$ denoting ``was item $i$ considered?''). We only observe a ranking from this instance if $X \ge k$. We can bound the probability that the consideration set meets this constraint:
  \begin{align}
    \Pr(X < k) &\le \Pr(X \le k) \le \left(\alpha e^{1-\alpha}\right)^{ k} \tag{by \Cref{lemma:chernoff}}, \nonumber
  \intertext{and therefore}
    \Pr(X \ge k) &= 1 - \Pr(X < k) \ge 1 - \left(\alpha e^{1-\alpha}\right)^{ k}. \label{eq:ast}
  \end{align}
  Also, note that if we observe $i$ in a length-$k$ ranking, it must have been the case that $X_i = 1$, and we have conditioned on $X \ge k$. Thus,
  \begin{align}
    \PLC(\R_{i \le k}) \le \Pr(X_i \mid X \ge k). \label{eq:dagger}
  \end{align}
  We can use this to lower bound $p_i $:
  \begin{align*}
    p_i &= \Pr(X_i)\\
    &= \Pr(X_i \mid X \ge k) \cdot \Pr(X \ge k) + \Pr(X_i \mid X < k) \cdot \Pr(X < k) \\
     &\ge \Pr(X_i \mid X \ge k) \cdot \Pr(X \ge k)\\
     &\ge \Pr(X_i \mid X \ge k) \cdot \left[1- \left(\alpha e^{1-\alpha}\right)^{ k}\right] \tag*{by \eqref{eq:ast}}\\
     &\ge \PLC(\R_{i \le k}) \cdot \left[1- \left(\alpha e^{1-\alpha}\right)^{ k}\right], \tag*{by \eqref{eq:dagger}}
  \end{align*}
  just as desired.
%    p_i &\ge \Pr(X_i \mid X \ge k) \left[1- \left(\alpha e^{1-\alpha}\right)^{ k}\right] \tag{by $(*)$}\\
%    &\ge \PLC(\R_{i \le k}) \cdot \left[1- \left(\alpha e^{1-\alpha}\right)^{ k}\right]. \tag*{\qedhere}
%  \end{align*}
\end{proof}

\ConsiderExactlyK*
\begin{proof}
  Let $X$ be the random variable representing the size of the consideration set. Then we can write
  \begin{align*}
    \sum_{C \subseteq \U, |C| = k} \PrC(C) &= \Pr(X = k \mid X \ge k)\\
    &= \frac{\Pr (X = k)}{\Pr(X \ge k)}\\
    &\le \frac{\Pr (X \le  k)}{\Pr(X > k)}\\
    &= \frac{\Pr (X \le  k)}{1 - \Pr (X \le  k)}\\
    &\le \textstyle \frac{\left(\alpha e^{1-\alpha}\right)^{ k}}{1 - \left(\alpha e^{1-\alpha}\right)^{ k}},
  \end{align*}
where the last inequality follows by \Cref{lemma:chernoff}.
\end{proof}

\IncreasePj*
\begin{proof}
We will be considering the probability that item $i$ is chosen first from a particular consideration set $C$. For brevity, define 
\begin{align*}
\Pr(i \mid C) &= \sum_{r \in \R_{i = 1}} \PL(r \mid C),
\end{align*}
summing over all rankings in which $i$ is chosen first.  (We could also have defined $\Pr(i \mid C)$ as $\exp(u_i) / \sum_{h \in C}\exp(u_h)$, by definition.)
  
First, we'll shrink $\PLC(\R_{i = 1})$ so it only involves a sum over certain consideration sets (excluding those where $|C \cup \set{j}| = k$):
\begin{equation}
  \begin{aligned}[b]
  \PLC(\R_{i = 1}) &= \sum_{r \in \R_{i = 1}} \sum_{C \subseteq \U} \PrC(C) \cdot \PL(r \mid C)\\
  &= \sum_{C \subseteq \U} \Big[\PrC(C) \sum_{r \in \R_{i = 1}} \PL(r \mid C)\Big]\\
  &= \sum_{C \subseteq \U} \PrC(C) \cdot \Pr(i \mid C)\\
  &\ge %\sum_{\substack{C \subseteq \U\\|C \cup \{j\}| > k\\i \in C}} \PrC(C) \sum_{r \in \R_{i = 1}} \PL(r \mid C)\\
  % &=
  \sum_{\substack{C \subseteq \U\\|C \cup \{j\}| > k\\i \in C}} \PrC(C) \cdot \Pr(i \mid C).
  \end{aligned}
  \label{eq:ub:decompose-pr[r=1]-into-summation}
\end{equation}
We partition the consideration sets in this summation into those that contain $j$ and those that don't, using the notation $\mathcal C_{i\cap j} = \set{{C \subseteq \U} \mid {i \in C}, {j \in C}}$ and $\mathcal C_{i\setminus j} = \set{{C \subseteq \U} \mid {i \in C}, {j \notin C}}$.  It will be helpful to pair each consideration set containing $j$ with its corresponding set excluding $j$, by defining $\phi: \mathcal C_{i \cap j} \rightarrow \mathcal C_{i \setminus j}$ where $\phi(C)= C \setminus \{j\}$. We can then rewrite the summation in \eqref{eq:ub:decompose-pr[r=1]-into-summation} as
\begin{align}
   &\sum_{\substack{C \subseteq \U\\|C \cup \{j\}| > k\\i \in C}} \PrC(C) \cdot \Pr(i \mid C) \nonumber\\
   ={}& \sum_{\substack{C \in \mathcal C_{i \cap j} \\|C| > k}} \PrC(C) \cdot \Pr(i \mid C) + \sum_{\substack{C \in \mathcal C_{i \setminus j}\\|C \cup \{j\}| > k}} \PrC(C) \cdot \Pr(i \mid C) \nonumber\\
%   ={}& \sum_{\substack{C \in \mathcal C_{i \cap j} \\|C| > k}} \PrC(C) \cdot \Pr(i \mid C) + \sum_{\substack{C \in \mathcal C_{i \cap j}\\|C| > k}} \PrC(\phi(C)) \cdot \Pr(i \mid \phi(C)) \nonumber\\
   ={}& \sum_{\substack{C \in \mathcal C_{i \cap j} \\|C| > k}} \left[\strut \PrC(C) \cdot\Pr(i \mid C) + \PrC(\phi(C)) \cdot \Pr(i \mid \phi(C)) \right].
  \label{eq:ub:combine-summation-using-phi}
\end{align}
  
Let $z_{k, p'}$ be the normalization constant in $\PrC$ after replacing $p_j$ with $p_j' > p_j$. Because $z_{k, p}$ denotes the probability that at least $k$ items are considered, we have $z_{k, p'} \ge z_{k, p}$. For any $C \in \mathcal C_{i \cap j}$ with $|C| > k$ (and note that this latter restriction is important for \eqref{eq:ub:lb-on-prc(phi(C))}, below: if $|C| \le k$, then $\PrC(\phi(C)) = 0$), we then have
\iffalse  
\begin{align*}
&\PrC(C) \Pr(i \mid C)  + \PrC(\phi(C)) \Pr(i \mid \phi(C)) \\
&=  \frac{\Pr(i \mid C)}{z_{k, p}}\left(\prod_{s \in C} p_s\right) \prod_{t \in \mathcal{U}\setminus C} (1-p_t)\\
&\quad+ \frac{\Pr(i \mid \phi(C))}{z_{k, p}}\left(\prod_{s \in \phi(C)} p_s\right) \prod_{t \in \mathcal{U}\setminus \phi(C)} (1-p_t)\\
&= \frac{\Pr(i \mid C)}{z_{k, p}}p_j\left(\prod_{s \in \phi(C)} p_s\right) \prod_{t \in \mathcal{U}\setminus C} (1-p_t)\\
&\quad+ \frac{\Pr(i \mid \phi(C))}{z_{k, p}}(1-p_j)\left(\prod_{s \in \phi(C)} p_s\right) \prod_{t \in \mathcal{U}\setminus C} (1-p_t)\\
&\ge \frac{\Pr(i \mid C)}{z_{k, p'}}p_j\left(\prod_{s \in \phi(C)} p_s\right) \prod_{t \in \mathcal{U}\setminus C} (1-p_t)\\
&\quad + \frac{\Pr(i \mid \phi(C))}{z_{k, p'}}(1-p_j)\left(\prod_{s \in \phi(C)} p_s\right) \prod_{t \in \mathcal{U}\setminus C} (1-p_t)
\end{align*}
\fi
\begin{align}
\PrC(C) 
&= \textstyle \frac{1}{z_{k, p}} \cdot \prod_{s \in C} p_s \cdot \prod_{t \in \mathcal{U}\setminus C} (1-p_t) \nonumber \\
&= \textstyle \frac{1}{z_{k, p}} \cdot p_j \cdot \prod_{s \in \phi(C)} p_s \cdot \prod_{t \in \mathcal{U}\setminus C} (1-p_t) \nonumber \\
  &\ge \textstyle \frac{1}{z_{k, p'}} \cdot p_j \cdot \prod_{s \in \phi(C)} p_s \cdot \prod_{t \in \mathcal{U}\setminus C} (1-p_t)
\label{eq:ub:lb-on-prc(C)}
\intertext{and, similarly,}
\PrC(\phi(C)) 
&= \textstyle \frac{1}{z_{k, p}} \cdot \prod_{s \in \phi(C)} p_s \cdot \prod_{t \in \mathcal{U}\setminus \phi(C)} (1-p_t) \nonumber \\
&= \textstyle \frac{1}{z_{k, p}} \cdot (1-p_j)  \cdot \prod_{s \in \phi(C)} p_s  \cdot \prod_{t \in \mathcal{U}\setminus C} (1-p_t) \nonumber\\
&\ge \textstyle \frac{1}{z_{k, p'}} \cdot (1-p_j)  \cdot \prod_{s \in \phi(C)} p_s  \cdot \prod_{t \in \mathcal{U}\setminus C} (1-p_t).
\label{eq:ub:lb-on-prc(phi(C))}    
\end{align}
Combining \eqref{eq:ub:lb-on-prc(C)} and \eqref{eq:ub:lb-on-prc(phi(C))}, still for $C \in \mathcal C_{i \cap j}$ with $|C| > k$, we have
\begin{equation}
  \begin{aligned}[b]
&\PrC(C) \cdot \Pr(i \mid C)  + \PrC(\phi(C))  \cdot \Pr(i \mid \phi(C)) \\
\ge {} & \textstyle \Pr(i \mid C) \cdot \frac{1}{z_{k, p'}} \cdot p_j \cdot \prod_{s \in \phi(C)} p_s \cdot \prod_{t \in \mathcal{U}\setminus C} (1-p_t)
  \\
& + \textstyle \Pr(i \mid \phi(C)) \cdot \frac{1}{z_{k, p'}} \cdot (1-p_j) \cdot \prod_{s \in \phi(C)} p_s \cdot \prod_{t \in \mathcal{U}\setminus C} (1-p_t).
\end{aligned}
\label{eq:ub:lb-on-guts-of-summation}
\end{equation}
Now note that \eqref{eq:ub:lb-on-guts-of-summation} has the form $p_j\cdot a + (1-p_j) \cdot b$, where $a < b$ (since $\Pr(i \mid C) < \Pr(i \mid \phi(C))$). Thus, when we replace $p_j$ with $p_j' > p_j$, the weighted average shrinks: we place more weight on the smaller term and less weight on the larger one. Using $\PrC'$ to denote consideration probabilities with $p_j$ replaced by $p_j'$, and making use of \eqref{eq:ub:decompose-pr[r=1]-into-summation} and \eqref{eq:ub:combine-summation-using-phi}, we then know:
\begin{equation}
  \begin{aligned}[b]
    &\PLC(\R_{i = 1})\\
    &\ge \sum_{\substack{C \in \mathcal C_{i \cap j} \\|C| > k}} \left[\PrC(C) \cdot \Pr(i \mid C) + \PrC(\phi(C)) \cdot \Pr(i \mid \phi(C)) \right]\\
    &\ge \sum_{\substack{C \in \mathcal C_{i \cap j} \\|C| > k}} \left[\PrC'(C) \cdot \Pr(i \mid C) + \PrC'(\phi(C)) \cdot \Pr(i \mid \phi(C)) \right]\\
    &= \sum_{\substack{C \subseteq \U\\|C \cup \{j\}| > k\\i \in C}} \PrC'(C) \cdot \Pr(i \mid C).
  \end{aligned}
  \label{eq:ub:assembled-lb-on-pr[r=1]}
\end{equation}
  Finally, we add the additive error term to recover the consideration sets missing from \eqref{eq:ub:assembled-lb-on-pr[r=1]}---those containing $j$ and with size exactly $k$---and arrive at $\PLC'(\R_{i = 1})$. Recall that $\sum_{i \in \U} p_i \ge \alpha k$; this inequality is still true after replacing $p_j$ with $p_j' > p_j$, so by \Cref{lemma:consider-exactly-k-bound}
  \begin{align}
    \sum_{C \subseteq \U, |C| = k} \PrC'(C) \le {\textstyle \frac{\left(\alpha e^{1-\alpha}\right)^{ k}}{1 - \left(\alpha e^{1-\alpha}\right)^{ k}}}.
    \label{eq:ub:missing-sets-are-rare}
  \end{align}
  Thus, by \eqref{eq:ub:assembled-lb-on-pr[r=1]} and \eqref{eq:ub:missing-sets-are-rare}, we have
  \begin{align*}
    & \PLC'(\R_{i = 1}) \\
    &= \sum_{\substack{C \subseteq \U\\|C \cup \{j\}| > k\\i \in C}} \PrC'(C) \cdot \Pr(i \mid C)  +  \sum_{\substack{C \subseteq \U\\|C| = k\\j \in C}} \PrC'(C) \cdot \Pr(i \mid C)\\
    &\le \sum_{\substack{C \subseteq \U\\|C \cup \{j\}| > k\\i \in C}} \PrC'(C) \cdot \Pr(i \mid C) +  \sum_{\substack{C \subseteq \U\\|C| = k}} \PrC'(C)\\
    &\le \sum_{\substack{C \subseteq \U\\|C \cup \{j\}| > k\\i \in C}} \PrC'(C) \cdot \Pr(i \mid C) +  {\textstyle \frac{\left(\alpha e^{1-\alpha}\right)^{ k}}{1 - \left(\alpha e^{1-\alpha}\right)^{ k}}}\\
    &\le \PLC(\R_{i = 1}) + {\textstyle \frac{\left(\alpha e^{1-\alpha}\right)^{ k}}{1 - \left(\alpha e^{1-\alpha}\right)^{ k}}},
  \end{align*}
  which is exactly the desired inequality. Finally, if at least $k+1$ of the consideration probabilities (including $p_j'$) are 1, then the sum in \eqref{eq:ub:missing-sets-are-rare} equals 0, since we can never have a size-$k$ consideration set. The same argument as above---identical up to \eqref{eq:ub:assembled-lb-on-pr[r=1]}, and with the tightened \eqref{eq:ub:missing-sets-are-rare}---then yields the inequality with no error term.
\end{proof}

\UBRepeated*

\begin{proof}
  We'll apply \Cref{lemma:increase-pj} to $\PLC(\R_{i=1})$ $n-1$ times, each time increasing $p_j$ to $1$ (for every $j \ne i$). Note that since we only increase consideration probabilities, $\alpha k$ remains a valid lower bound on $\sum_{i \in \U}p_i$. The first $k$ times we do this, we incur the additive error from \Cref{lemma:increase-pj}. After $k$ consideration probabilities have been increased to 1, we no longer incur additional error from subsequent increases. Thus, using $\PLC'(\R_{i=1})$ to denote the probability that $i$ is ranked first if all consideration probabilities except $p_i$ are 1,
\begin{align}
  \PLC'(\R_{i=1}) &\le \PLC(\R_{i=1}) + k {\textstyle \frac{\left(\alpha e^{1-\alpha}\right)^{ k}}{1 - \left(\alpha e^{1-\alpha}\right)^{ k}}}.
  \label{eq:i-wins-when-everything-else-always-considered}
\end{align}
Now, we can simplify $\PLC'(\R_{i=1})$, since $\U$ (occurring w.p.\ $p_i$) is the only feasible consideration set that includes $i$, since all other consideration probabilities equal 1:
\begin{align*}
  \PLC'(\R_{i=1}) = \sum_{r\in\R_{i=1}}p_i\cdot\PL(r\mid \U)
  &= p_i\sum_{r\in\R_{i=1}}\PL(r\mid \U)\\
  &= p_i \cdot  \frac{\exp(u_i)}{\sum_{j \in \U}\exp(u_j)}.
\end{align*}
We can combine this with \eqref{eq:i-wins-when-everything-else-always-considered} and solve for $p_i$:
\begin{align*}
  &p_i \cdot  \frac{\exp(u_i)}{\sum_{j \in \U}\exp(u_j)} \le \PLC(\R_{i=1}) + k {\textstyle \frac{\left(\alpha e^{1-\alpha}\right)^{ k}}{1 - \left(\alpha e^{1-\alpha}\right)^{ k}}}\\
\Rightarrow{}& p_i \le \frac{\sum_{j \in \U}\exp(u_j)}{\exp(u_i)} \cdot \left(\PLC(\R_{i=1}) + k {\textstyle \frac{\left(\alpha e^{1-\alpha}\right)^{ k}}{1 - \left(\alpha e^{1-\alpha}\right)^{ k}}}\right). \tag*{\qedhere}
\end{align*}
\end{proof}

\subsection{Algorithms for Lower and Upper Bounds}
\label{sect:supplement:algorithms}

Here is the deferred proof of \Cref{thm:lower-bound-alg-correctness}, and the omitted upper-bound algorithm (and its proof of correctness).

\LBcorrectness*
\begin{proof}
  We begin with correctness. We'll show that $b_i \le p_i$ is an invariant of the nested for loop for every $i \in \U$. First, the initial values of $b_i$ satisfy $b_i \le p_i$ by \Cref{thm:consideration-initial-LB}. Now, suppose inductively that $b_j \le p_j$ for all $j \in \U$ before an iteration of the innermost loop. We only update $b_j$ if $c = \PLC(\R_{i \le \ell}) / \PLC(\R_{j \le \ell})$. Additionally, $j$ must be an outneighbor of $i$ in $G$. By definition of $G$, we then have $u_i > u_j$. Thus, in any case where $b_j$ is updated, we know from \Cref{thm:consider-gap} that $\smash{\frac{p_i}{1 - p_i}} \le c \cdot \smash{\frac{p_j}{1-p_j}}$. We can rewrite this as a lower bound on $p_j$, as in \eqref{eq:lower-bound-on-pj}: $$ p_j \ge \frac{p_i}{c - cp_i + p_i} .$$ The right side of this lower bound is an increasing function of $p_i$ for $p_i, c \in (0, 1]$, so we can replace $p_i$ with a smaller value (namely, $b_i$, which is smaller by inductive hypothesis) and maintain the inequality. Thus, we know $\smash{p_j \ge \frac{b_i}{c - cb_i + b_i}}$. Therefore, if we update $b_j$ to be equal to this lower bound in the loop, we still maintain the invariant $b_j \le p_j$. By induction, we then have $b_i \le p_i$ for all $i \in \U$ when the algorithm terminates.
  
 Concerning the runtime, initializing the lower bounds takes time $O(n)$ and constructing $G$ takes times $O(k n^2)$ (by looping through all pairs of items and all $k$ values of $\ell$). Additionally, topologically sorting $G$ takes time $O(n^2)$ and the nested loop takes time $O(kn^2)$, for a total running time of $O(kn^2)$. 
\end{proof}

\begin{theorem}
  Suppose $\sum_{i \in \U} p_i \ge  \alpha k $. \Cref{alg:upper-bounds} returns upper bounds $b_i$ for each $i \in \U$ such that $p_i \le b_i$ in time $O(kn^2)$.
\end{theorem}
\begin{proof}
The analysis is nearly identical to \Cref{alg:lower-bounds}, so we only highlight the differences. First, the initial values of $b_i$ are valid upper bounds by \Cref{thm:upper-bound:repeated-applications}. Then, we use \Cref{thm:consider-gap} in its upper bound form \eqref{eq:upper-bound-on-pi}: $\smash{p_i \le \frac{c p_j}{1 - p_j + cp_j}}$. This bound is an increasing function of $p_j$, so we can replace $p_j$ by a larger quantity (namely, an upper bound $b_j \ge p_j$) and maintain the inequality. Whenever we update $b_i$, the conditions of \Cref{thm:consider-gap} are satisfied, so $b_i$ remains a valid upper bound on $p_i$. Correctness then follows by induction and the runtime analysis is the same as for \Cref{alg:lower-bounds}.
\end{proof}

\begin{algorithm}[tb]
   \caption{Upper-bounding consideration probabilities.}
   \label{alg:upper-bounds}
\begin{algorithmic}[1]
\algsetup{
  indent=1.5em,
  linenosize=\scriptsize}
    \REQUIRE items $\U$, (estimates of) utilities $u_i$, (estimates of) top-$\ell$ ranking probabilities $\PLC(\R_{i \le \ell})$, lower bound $\alpha k$ on expected consideration set size ($\alpha > 1$)
    \STATE $b_i \gets \frac{\sum_{j \in \U}\exp(u_j)}{\exp(u_i)} \cdot \left(\PLC(\R_{i=1}) + {\textstyle \frac{k \left(\alpha e^{1-\alpha}\right)^{ k}}{1 - \left(\alpha e^{1-\alpha}\right)^{ k}}}\right) $ for all $i \in \U$  
    \STATE $E\gets \set{\vcenter{\hbox{$\displaystyle \tup{j, i} \subset \U \times \U \mid
          %\begin{array}[t]{@{}l@{}}
            u_i > u_j \text{ and } %\\ \quad
            \exists \ell : \PLC(\R_{i \le \ell}) < \PLC(\R_{j \le \ell})
          %\end{array}
      $}}}$
    \STATE $G \gets \tup{\U, E}$
    \FORALL {$j$ in a topological sort of $G$}
      \FORALL{out-neighbors $i$ of $j$ in $G$}
        \FOR{$\ell  = 1, \dots, k$}
        \STATE $c \gets \PLC(\R_{i \le \ell}) / \PLC(\R_{j \le \ell})$
          \IF {$c \le 1$}
            \STATE $b_i \gets \min\left\{b_i, \frac{cb_j}{1 - b_j + cb_j}\right\}$
          \ENDIF
        \ENDFOR
      \ENDFOR
    \ENDFOR
    \RETURN $b_i$ for each $i \in \U$ such that $p_i \le b_i$
  \end{algorithmic}
\end{algorithm}

\subsection{Algorithms for approximating PL+C probabilities}
\label{app:approx-alg}

\begin{algorithm}[h]
  \caption{Probably approximately computing PL+C probabilities.}
  \label{alg:probapprox-plc}
\begin{algorithmic}[1]
\algsetup{
 indent=1.5em,
 linenosize=\scriptsize}
   \REQUIRE ranking $r$, utilities $u_i > 0$ and consideration probabilities $p_i$ for all $i \in \U$, error tolerance $\epsilon$, failure probability $\delta$
   \STATE $\hat R \gets 0$
   \STATE $s \gets \lceil\log(4/\delta)/(2\epsilon^2) \rceil$
   \FOR{$s$ iterations}
     \STATE $C \gets \emptyset$
     \WHILE{$|C| < k$}
        \STATE sample consideration set $C$ according to consideration probabilities $p_i$
     \ENDWHILE
     \STATE $\hat R \gets \hat R + \PL(r | C)$
   \ENDFOR
   \STATE $\hat R \gets \hat R / s$
   \RETURN $\hat R$
 \end{algorithmic}
\end{algorithm}

\ProbApproxalg*
\begin{proof}
  Let $R$ be the random variable that takes value $\PL(r|C)$ with probability $\PrC(C).$ It follows by \eqref{eq:PLC} that $\PLC(r) = \E[R].$ Let $\hat R = \frac{1}{s} \sum_{i = 1}^s R$ be the empirical mean of $s$ i.i.d.~samples from $R$. As $R\in[0,1]$ since it outputs probabilities, by Hoeffding's inequality~\citep{mitzenmacher2017probability}, we get
  \begin{align*}
    \Pr(|\hat R - \PLC(r)| \ge \epsilon) \le 2\exp(-2s\epsilon^2).
  \end{align*}
  Setting $\delta/2 = 2\exp(-2s\epsilon^2)$ and solving for $s$, we find that we need $s = \log(4/\delta)/(2\epsilon^2) = O(\log(1/\delta)/\epsilon^2)$ samples from $R$ to estimate $\PLC(r)$ with $\epsilon$-additive error and failure probability at most $\delta/2$.

In order to sample from $R$, \Cref{alg:probapprox-plc} uses rejection sampling to discard cases where $|C| < k$. We can bound the number of consideration sets $T$ we need to sample to get $s$ consideration sets of size at least $k$---and therefore $s$ useful samples from $R$. The probability of drawing a sufficiently large consideration set is exactly $z_{k, p}$, so $T$ is distributed according to a negative binomial distribution (counting the total number of trials, not just failures), namely $T\sim \text{NB}(s, z_{k, p})$, which has expectation $s /z_{k, p}$.  Using a tail bound for this type of negative binomial distribution~\citep{brown2011wasted}, we have that for any $c > 1$, 
  \begin{align*}
    \Pr(T > c\E[T]) \le \exp(-cs(1-1/c)^2/2).
  \end{align*}
  We want to find a $c$ such that this failure probability is at most $\delta / 2$, since then if we run $cs/z_{k,p}$ times, we will only fail to get $s$ successes with probability at most $\delta / 2$. Setting $\delta/2 \ge \exp(-cs(1-1/c)^2/2)$ and solving for $c$, we find that we need
  \begin{align*}
    &c+1/c \ge \frac{2\log(2/\delta)}{s}+2.
  \end{align*}
  Thus $c = \frac{2\log(2/\delta)}{s}+2$ suffices. That is, if we sample $\left(\frac{2\log(2/\delta)}{s}+2\right)s / z_{k,p}= O((\log(1/\delta)+s) / z_{k,p}) $ sets, then we will have at least $s$ sufficiently large sets with probability at least $1-\delta/2$. We can plug in our previously found value of $s$ to find that the total number of iterations we need is $O((\log(1/\delta)+\log(1/\delta)/\epsilon^2) / z_{k,p}) = O(\log(1/\delta)/(z_{k,p}\epsilon^2))$.

 By the union bound, one or both of the approximation and runtime guarantees fail with probability at most $\delta$ (either one fails w.p.\ at most $\delta / 2$), so our overall success probability is at least $1-\delta$.
  
  Lastly, it takes $O(n)$ time to sample one consideration set, and $O(kn)$ time to compute $\PL(r|C)$ from one consideration set $C$, so the overall runtime (with probability at least $1-\delta$) is $O(kn\log(1/\delta)/(z_{k,p}\epsilon^2))$.
\end{proof}

Next, we provide our deterministic algorithm for approximating PL+C probabilities. The algorithm clusters possible consideration sets into bins according to their total exp-utility and stores the total consideration probability of the sets in each bin (without storing the sets themselves). These sets are built up item-by-item, each time adding the consideration probability of new sets to the appropriate bin. At the end of the algorithm, we know the approximate exp-utility of sets in each bin and their exact consideration probabilities. This allows us to estimate the PL+C probability of a ranking $r$. Since we know all items in $r$ must have been considered, the algorithm only builds up consideration sets over the items $\overline r = \U \setminus r$. Finally, we note that the normalizing constant $z_{k, p}$ can actually be computed efficiently. The \textsf{DC} (``direct computation'') algorithm~\citep{biscarri2018simple} we use as a subroutine in \Cref{alg:approx-plc} computes the full PMF of the Poisson binomial distribution in time $O(n^2)$, which we can sum over to get $z_{k, p}$.

\begin{algorithm}[h!]
   \caption{Approximately computing PL+C probabilities.}
   \label{alg:approx-plc}
\begin{algorithmic}[1]
\algsetup{
  indent=1.5em,
  linenosize=\scriptsize}
    \REQUIRE ranking $r$, utilities $u_i > 0$ and consideration probs.\ $p_i$ for all $i \in \U$, error tolerance $\epsilon$
    \IF{$k = n$}
      \RETURN $\PL(r \mid \U)$
    \ENDIF
    \STATE $\delta \gets \epsilon / (2kn)$
    \STATE $\overline r \gets \U \setminus r$
    \STATE $s \gets \lfloor \log_{1 + \delta}\left(\sum_{i \in \overline r}e^{u_i}\right) \rfloor$
    \STATE $A\gets $ array of size $s + 1$ indexed $-1, 0, \dots, s$
    \STATE $A[-1] \gets (1, 0)$
    \STATE $A[0, \dots, s] \gets (0, 0)$
    \STATE $A' \gets A$
    \FORALL{$i \in \overline r$}
      \STATE $A'[j] \gets (A[j][0] \cdot (1-p_i), A[j][1])$ (for every $j = -1, 0, \dots, s$)
      \FOR{$j = -1, \dots, s$}
        \IF{$A[j][0] > 0$}
          \STATE $h_j \gets \lfloor \log_{1+\delta}(A[j][1] + e^{u_i})\rfloor$
          \STATE $A'[h_j] \gets (A'[h_j][0] +  A[j][0] \cdot p_i, \max\{A'[h_j][1], A[j][1] + e^{u_i}\})$
        \ENDIF
      \ENDFOR
      \STATE $A \gets A'$
    \ENDFOR
    \STATE $z_{k, p} \gets \sum_{i = k}^n \textsf{DC}(p_1, \dots, p_n)[i]$
    \RETURN $\frac{\prod_{i \in r} p_i}{z_{k,p}} \sum_{j = -1}^s A[j][0] \cdot \prod_{i =1}^k \frac{e^{u_i}}{(1+\delta)^{j} + \sum_{a \in \{r_{i}, \dots, r_k\}} e^{u_{a}}}$
  \end{algorithmic}
\end{algorithm}

\begin{lemma}\label{lemma:approx-alg}
  After processing items $\{1, \dots, i\}\subseteq \overline r$ in the outer for loop, there is a partition $\{C_{i, -1}, \dots, C_{i, s}\}$  of $\mathcal P(\{1, \dots, i\})$ such that:
  \begin{enumerate}
    \item $C_{i, -1} = \{\emptyset\}$.
    \item For all $-1 \le j \le s$,  $A[j][0] = \sum_{C \in C_{i,j}}\left(\prod_{a \in C}p_a\right) \left(\prod_{b \in \{1, \dots, i\}\setminus C}(1-p_b)\right)$. 
    \item For every $0 \le j \le s$, every $C \in C_{i,j}$ satisfies $(1+\delta)^{j - i}\le \sum_{a \in C}e^{u_a} < (1+\delta)^{j + 1}$. 
    \item  $A[j][1] = \max_{C \in C_{i, j}}\sum_{a \in C}e^{u_a}$.
  \end{enumerate}
  \end{lemma}
  
  \begin{proof}
    By induction on $i$. 
    
    Base case: $i = 0$. Before processing any items, define the partition to be $C_{0, -1} = \{\emptyset\}$ and every other $C_{0, j} = \emptyset$, satisfying claim 1. By lines 7 and 8, we have $A[-1][0] = 1$ and every other $A[j][0] = 0$, satisfying claim 1. Claim 3 is vacuously true, since all the $C_{0, j}$ with $0 \le j \le s$ are empty. Finally, claim 4 also holds since all the $A[j][1] = 0$ (either $C = \emptyset$ or $C_{0, j} = \emptyset$).
    
    Inductive case: $i \ge 1$. By inductive hypothesis, the $C_{i-1, j}$ form a partition of $\mathcal P (\{1, \dots, i -1\})$. Thus, every set in $\mathcal P (\{1, \dots, i\})$ is either in some $C_{i-1, j}$ or is obtained by adding item $i$ to a set in some $C_{i-1, j}$. To make the new partition, we'll start with the old one and, for each $C_{i-1, j}$, add all of its sets plus $i$ to the collection $C_{i, h_j}$, with $h_j$ defined in line 14. We can verify that $h_j$ is always a valid index:
    \begin{itemize}
      \item Since the utilities are all positive and $A[j][1]$ is non-negative, $ \lfloor \log_{1+\delta}(A[j][1] + e^{u_i})\rfloor \ge 0$.
      \item Since $A[j][1] $ is the sum of exp-utilities of some set $C$ in $C_{i-1, j}$ by inductive hypothesis (of claim 4), adding a new exp-utility cannot cause $h$ to exceed $s$, which is defined using the sum of all item exp-utilities.
    \end{itemize}
      Formally, the new partition $C_{i,j}$ can be written as follows: 
    \begin{align*}
      C_{i,j} = C_{i-1,j} \cup \{C \cup \{i\}, \text{ for all } C \in C_{i-1, \ell}, \text{ for all } \ell \text{ s.t.\ } h_\ell =j\}.
    \end{align*}
    
  In this new partition, claim 1 still holds: since $h_j \ge 0$, we leave $C_{i, -1}$ unchanged. We can show claim 2 still holds as well. Just before executing line 16, lines 11 and 15 have resulted in:
    \begin{align*}
      A'[j][0] &= A[j][0] \cdot (1-p_i) + \sum_{\ell : h_\ell = j, A[\ell][0] > 0} A[\ell][0] \cdot p_i.
    \end{align*}
    (Note that the condition $A[\ell][0] > 0$ prevents us from adding $i$ to empty collections of sets.) By inductive hypothesis and our construction of $C_{i,j}$,  
        \begin{align*}
      A'[j][0] &= \sum_{C \in C_{i-1,j}}\left(\prod_{a \in C}p_a\right) \left(\prod_{b \in \{1, \dots, i-1\}\setminus C}(1-p_b)\right) \cdot (1-p_i) \\
      &\quad+ \sum_{\substack{\ell : h_\ell = j,\\ A[\ell][0] > 0}} \sum_{C \in C_{i-1,\ell}}\left(\prod_{a \in C}p_a\right) \left(\prod_{b \in \{1, \dots, i-1\}\setminus C}(1-p_b)\right) \cdot p_i\\
      &= \sum_{C \in C_{i,j}, i \notin C}\left(\prod_{a \in C}p_a\right) \left(\prod_{b \in \{1, \dots, i\}\setminus C}(1-p_b)\right)  \\
      &\quad+ \sum_{C \in C_{i,j}, i \in C}\left(\prod_{a \in C}p_a\right) \left(\prod_{b \in \{1, \dots, i\}\setminus C}(1-p_b)\right)\\
      &= \sum_{C \in C_{i,j}}\left(\prod_{a \in C}p_a\right) \left(\prod_{b \in \{1, \dots, i\}\setminus C}(1-p_b)\right).
    \end{align*}
    Thus claim 2 is satisfied at the end of the for loop iteration. Claim 4 follows from our construction of $C_{i,j}$, inductive hypothesis, and lines 11 and 15: whenever we add new sets to $C_{i, j}$ with higher exp-utility, we also update $A[j][1]$ to be the larger of the new and old exp-utility sums. Finally, we show claim 3. 
    
    Any set $C \in C_{i, j}$ was either (a) in $C_{i-1, j}$ or (b) was just added by including $i$. For carried-over sets, we have $(1+\delta)^{j - (i-1)}\le \sum_{a \in C}e^{u_a} < (1+\delta)^{j + 1}$ by inductive hypothesis, so the claim still holds with $i$. For newly added sets, we know they are of the form $C = C' \cup \{i\}$, where $C' \in C_{i-1, \ell}$ such that $h_\ell = j$. 
    
    If $C' = \emptyset$, then $j = h_\ell = \lfloor \log_{1+\delta}(e^{u_i})\rfloor$. Thus $\log_{1+\delta}e^{u_i} - 1 < j$ and $j \le \log_{1+\delta}e^{u_i}$. We can rewrite these inequalities as $e^{u_i} < (1+\delta)^{ j + 1} $ and $(1+\delta)^j \le e^{u_i}$. Since $i \ge 0$, this shows claim 3 is satisfied.
    
    Now suppose $C' \ne \emptyset$. By inductive hypothesis (claim 3), we then have:
    \begin{align*}
      (1+\delta)^{\ell - (i-1)}\le \sum_{a \in C'}e^{u_a} < (1+\delta)^{\ell + 1}
    \end{align*}
  
    We also know by inductive hypothesis (claim 4) that $A[\ell][1] = \max_{C \in C_{i-1, j}}\sum_{a \in C}e^{u_a}$. $C'$ is one such $C$, so we have:
    \begin{align*}
        (1+\delta)^{\ell - (i-1)}\le \sum_{a \in C'}e^{u_a} \le A[\ell][1] < (1+\delta)^{\ell + 1}
    \end{align*}
Using these bounds, we can then bound how much larger $A[\ell][1]$ is than the exp-utility sum of $C'$:
    \begin{align*}
      &\frac{A[\ell][1]}{\sum_{a \in C'} e^{u_a}} < \frac{(1+\delta)^{\ell+1}}{(1+\delta)^{\ell - i + 1}} = (1+\delta)^{i}\\
      \Rightarrow \quad & A[\ell][1] < (1 + \delta)^i \sum_{a \in C'} e^{u_a}.
    \end{align*}
    We will now use this fact to establish the bounds in claim 3.  For the lower bound:
      \begin{align*}
      &j = h_\ell = \lfloor \log_{1+\delta}(A[\ell][1] + e^{u_i})\rfloor  \\
      \Rightarrow\quad & j \le \log_{1+\delta}((1 + \delta)^i \sum_{a \in C'}e^{u_a} + e^{u_i}) \tag{IH, claim 4}\\
      \Rightarrow\quad & (1+\delta)^j \le (1 + \delta)^i \sum_{a \in C'}e^{u_a} + e^{u_i}\tag{exponentiate}\\
        \Rightarrow\quad & (1+\delta)^j \le (1 + \delta)^i \sum_{a \in C'}e^{u_a} + (1+\delta)^i e^{u_i}\tag{increase RHS}\\
      \Rightarrow\quad & (1+\delta)^j \le (1 + \delta)^i \sum_{a \in C}e^{u_a} \tag{fold into sum}\\
          \Rightarrow\quad & (1+\delta)^{j - i} \le \sum_{a \in C}e^{u_a}.\\
    \end{align*}
    For the upper bound:
    \begin{align*}
      &j = h_\ell = \lfloor \log_{1+\delta}(A[\ell][1] + e^{u_i})\rfloor  \\
      \Rightarrow\quad & j+1 > \log_{1+\delta}(A[\ell][1] + e^{u_i})\\
      \Rightarrow\quad & (1+\delta)^{j+1} > A[\ell][1] + e^{u_i} \ge \sum_{a \in C'}e^{u_a} + e^{u_i} = \sum_{a \in C} e^{u_a}.\\
    \end{align*}
    Thus all of the claims still hold after iteration $i$. 
  \end{proof}

We can now prove \Cref{alg:approx-plc} is correct.
\Approxalg*
\begin{proof}
  By \Cref{lemma:approx-alg}, at the end of the nested loops of \Cref{alg:approx-plc}, there is a partition $\{C_{-1}, \dots, C_{s}\}$ of the subsets of $\overline r$ such that $A[j][0]$ stores the total consideration probabilities of the sets in $C_j$ and such that the exp-utility sums of every set in $C_j$ is between $(1+\delta)^{j - n}$ and $(1+\delta)^{j+1}$ (since $|\overline r| \le n $).  Thus,
  \begin{align*}
   &\frac{\prod_{i \in r} p_i}{z_{k,p}} \sum_{j = -1}^s A[j]  \prod_{i =1}^k \frac{e^{u_i}}{(1+\delta)^{j} + \sum_{a \in \{r_{i}, \dots, r_k\}} e^{u_{a}}}\\
   &=\frac{\prod_{i \in r} p_i}{z_{k,p}} \sum_{j = -1}^s \sum_{C \in C_{j}}\left(\prod_{a \in C}p_a\right) \left(\prod_{b \in \overline r \setminus C}(1-p_b)\right)  \prod_{i =1}^k \frac{e^{u_i}}{(1+\delta)^{j} + \sum_{a \in \{r_{i}, \dots, r_k\}} e^{u_{a}}}\\
      &= \sum_{j = -1}^s \sum_{C \in C_{j}}\frac{1}{z_{k,p}}\left(\prod_{a \in C\cup r}p_a\right) \left(\prod_{b \in \overline r \setminus C}(1-p_b)\right)  \prod_{i =1}^k\frac{e^{u_i}}{(1+\delta)^{j} + \sum_{a \in \{r_{i}, \dots, r_k\}} e^{u_{a}}}\\
      &= \sum_{j = -1}^s \sum_{C \in C_{j}}\PrC(C\cup r)  \prod_{i =1}^k\frac{e^{u_i}}{(1+\delta)^{j} + \sum_{a \in \{r_{i}, \dots, r_k\}} e^{u_{a}}}.
  \end{align*}
  Note that the product is very close to the desired Plackett--Luce probability:
    \begin{align*}
  \PL(r \mid C\cup r) &=  \prod_{i =1}^k\frac{e^{u_i}}{\sum_{a \in C\cup \{r_{i}, \dots, r_k\}} e^{u_{a}}}\\
  &= \prod_{i =1}^k\frac{e^{u_i}}{\sum_{a \in C} e^{u_{a}} + \sum_{a \in \{r_{i}, \dots, r_k\}} e^{u_{a}}}
\end{align*}

  We'll show that the error introduced by estimating the exp-utility of $C$ by $(1+\delta)^{j}$ is not too large. Consider the size of the maximum possible underestimation factor (when $\sum_{a \in C} e^{u_{a}}$ achieves its upper bound $(1+\delta)^{j +1}$). Let $e_{i} = \sum_{a \in \{r_{i}, \dots, r_k\}} e^{u_{a}}$
\begin{align*}
  \frac{e^{u_i}}{(1+\delta)^{j} + e_i} / \frac{e^{u_i}}{(1+\delta)^{j+1} + e_i} &= \frac{(1+\delta)^{j +1} + e_i}{(1+\delta)^{j} + e_i}\\
  &\le \frac{(1+\delta)^{j +1}}{(1+\delta)^{j}}\\
  &= (1+\delta).
  \end{align*}
  Now consider the maximum possible overestimation factor (when $\sum_{a \in C} e^{u_{a}}$ achieves its lower bound $(1+\delta)^{j -n}$):
\begin{align*}
  \frac{e^{u_i}}{(1+\delta)^{j-n} + e_i} / \frac{e^{u_i}}{(1+\delta)^{j} + e_i} &= \frac{(1+\delta)^{j} + e_i}{(1+\delta)^{j-n} + e_i}\\
  &\le \frac{(1+\delta)^{j}}{(1+\delta)^{j-n}}\\
  &= (1+\delta)^n.
  \end{align*}

  Thus, combining the largest possible under- and overestimation errors for the PL probability (incurred $k$ times, once for each term of the product),
  \begin{align*}
    &\sum_{j = -1}^s \sum_{C \in C_{j}}\PrC(C\cup r)  \PL(r\mid r \cup C) (1+\delta)^{-k}\\
    &\le \sum_{j = -1}^s \sum_{C \in C_{j}}\PrC(C\cup r)  \prod_{i =1}^k\frac{e^{u_i}}{(1+\delta)^{j} + \sum_{a \in \{r_{i}, \dots, r_k\}} e^{u_{a}}}\\
    &\le \sum_{j = -1}^s \sum_{C \in C_{j}}\PrC(C\cup r)  \PL(r\mid r \cup C) (1+\delta)^{nk}.
  \end{align*}
  Rewriting the sums, since we know the $C_j$ partition the subsets of $\overline r$,
    \begin{align*}
    & (1+\delta)^{-k} \cdot \PLC(r) = \sum_{C \subseteq \U}\PrC(C)  \PL(r\mid C) (1+\delta)^{-k}\\
    &\le \sum_{j = -1}^s \sum_{C \in C_{j}}\PrC(C\cup r)  \prod_{i =1}^k\frac{e^{u_i}}{(1+\delta)^{j} + \sum_{a \in \{r_{i}, \dots, r_k\}} e^{u_{a}}}\\
    &\le \sum_{C \subseteq \U}\PrC(C)  \PL(r\mid C) \PL(r\mid r \cup C) (1+\delta)^{nk} = (1+\delta)^{nk}\cdot  \PLC(r).
  \end{align*}
    Since $e^x \ge (1 + x / n)^n$ for $x \le n$, we have that $(1 + \delta)^{nk} = (1 + \epsilon/(2nk))^{nk} \le e^{\epsilon/2} \le 1 + \epsilon $ (the final inequality follows from the identity $e^x \le 1 + x + x^2$ for $x \le 1$). We therefore return an estimate $\hat p$ satisfying:
    \begin{align*}
      \frac{\PLC(r)}{1+\epsilon} \le \hat p \le  (1+\epsilon)\cdot  \PLC(r).
    \end{align*}
    
  Finally, we establish the runtime of the algorithm. The outer loop runs $n - k = O(n)$ times. Line 11 takes $O(s)$ time, the inner loop runs $O(s)$ times, and lines 13-15 take constant time per iteration. Thus the full nested loop takes $O(ns)$ time. The \textsf{DC} algorithm takes time $O(n^2)$. Line 18 takes $O(ns)$ time, for a total runtime of $O(n^2 + ns)$. Now we can bound $s$ :
  \begin{align*}
    s &= \left\lfloor \log_{1+\delta} \sum_{i \in \overline r} e^{u_i}\right\rfloor\\
    &\le \frac{\ln \sum_{i \in \overline r} e^{u_i}}{\ln (1+ \delta)}\\
     &\le (1 + \delta)\frac{\ln \sum_{i \in \overline r} e^{u_i}}{\delta}\tag{since $\ln (1+x) \ge \frac{x}{1+x}$ for $x > -1$}\\
     &= \frac{(2kn)\ln \sum_{i \in \overline r} e^{u_i}}{\epsilon} + \ln \sum_{i \in \overline r} e^{u_i}\tag{since $\delta = \epsilon / (2kn)$}\\
     &\le \frac{(2kn)(\max_{i \in \overline r}u_i + \ln n)}{\epsilon} + (\max_{i \in \overline r}u_i + \ln n) \tag{property of LogSumExp}
  \end{align*}
Thus $s = O(kn\max\{m, \log n\}/\epsilon)$ and the claimed runtime follows.
\end{proof}

\section{ADDITIONAL FIGURES}

\begin{figure}
\centering
\scalebox{0.75}{
  \begin{tikzpicture}[x=0.4\textwidth,y=2.6mm,yscale=-1]
      \tikzset{
        end point/.style = {fill=black, inner sep=0.125pc, circle},
        baseline end point/.style = {fill=blue!25, inner sep=0.2pc, rectangle},
        inner line/.style = {line width=0.125pc,black},
        outer line/.style = {line width=0.25pc,blue!25},
        state name/.style = {fill=white, inner xsep=0.0625pc, inner ysep=0pc, text=black, font=\small}
      }
      \draw (0,0) -- (0,50) ;
      \draw (1,0) -- (1,50) ;
      \draw[very thin, font=\small] (0,0) grid[xstep=0.2, ystep=50] (1,50);
      \foreach \i in {0.0, 0.2, 0.4, 0.6, 0.8, 1.0} {
        \node[above=0pc of {\i,0}, gray, overlay] {\i};
      };
      
      \begin{filecontents*}{data/bounds_alpha2.csv}
,state,lower,lower_baseline,upper,upper_baseline
45,Virginia,0.37830610469578824,0.28960103869437426,1.0,30.17338319531387
20,Massachusetts,0.3769396842840257,0.30022668831553473,1.0,32.801528417189665
31,New York,0.3093939154396732,0.3093939154396732,1.0,26.582681455097255
46,Washington,0.2509919459744336,0.02812671958542484,1.0,168.37881071618108
37,Pennsylvania,0.24876520877775743,0.24876520877775743,1.0,31.19865559910549
7,Delaware,0.1388886700292714,0.09521415444843814,1.0,75.08431863771425
4,California,0.11104845584467732,0.11104845584467732,1.0,54.614160239696545
42,Texas,0.06396224379796611,0.06396224379796611,1.0,47.715294323907955
44,Vermont,0.059116718871394615,0.009375573195141613,1.0,132.13900874156357
19,Maryland,0.056878477383859115,0.056878477383859115,1.0,66.65757532889944
47,West Virginia,0.049754149403322435,0.015625955325236023,1.0,125.18876531594125
34,Ohio,0.04500264127675296,0.0077088046271164366,1.0,119.97497903818416
0,Alabama,0.04020337648266525,0.011042341763166787,1.0,137.79176890035956
18,Maine,0.0353556313922879,0.009792265337147906,1.0,153.12806086755276
16,Kentucky,0.035355631392287895,0.008542188911129024,1.0,134.5053446695641
17,Louisiana,0.03333537136050351,0.03333537136050351,1.0,69.06828404067426
6,Connecticut,0.032710333147494074,0.032710333147494074,1.0,76.08942410722743
28,New Hampshire,0.030771091350688878,0.01437587889921714,1.0,100.58052493085783
38,Rhode Island,0.03045866768300399,0.01229241818918567,1.0,149.4164661529016
32,North Carolina,0.024337764541824413,0.016459339609248606,1.0,82.16084221519196
9,Georgia,0.024337764541824413,0.017084377822258048,1.0,83.73211562323506
35,Oklahoma,0.02051405538729343,0.0027084989230409104,1.0,180.1414328838998
13,Indiana,0.02051405538729343,0.002916844994044057,1.0,179.31621572800367
39,South Carolina,0.02000122281630211,0.02000122281630211,1.0,77.94738102715543
29,New Jersey,0.018164054050115827,0.010000611408151054,1.0,115.62061386975428
49,Wyoming,0.015464853254524994,0.0012500764260188818,1.0,255.49964705898103
3,Arkansas,0.015464853254524994,0.0018751146390283224,1.0,197.62209576764266
10,Hawaii,0.015464853254524994,0.005208651775078673,1.0,221.71357629963208
8,Florida,0.014584224970220288,0.014584224970220288,1.0,111.49323727200967
21,Michigan,0.010363324601078224,0.0031251910650472044,1.0,151.3440670046754
1,Alaska,0.010363324601078224,0.003750229278056645,1.0,267.5467883288691
5,Colorado,0.010363324601078224,0.0031251910650472044,1.0,181.29850186864985
23,Mississippi,0.0102089574791542,0.0102089574791542,1.0,100.07639965349749
12,Illinois,0.009375573195141613,0.009375573195141613,1.0,104.9528010777625
41,Tennessee,0.005208651775078673,0.005208651775078673,1.0,116.997541875552
36,Oregon,0.0031251910650472044,0.0031251910650472044,1.0,184.132610599341
24,Missouri,0.0027084989230409104,0.0027084989230409104,1.0,131.70545984633506
14,Iowa,0.0025001528520377635,0.0025001528520377635,1.0,219.98968689580414
2,Arizona,0.0020834607100314695,0.0020834607100314695,1.0,224.90290426384175
15,Kansas,0.0014584224970220286,0.0014584224970220286,1.0,201.43828077331537
33,North Dakota,0.0012500764260188818,0.0012500764260188818,1.0,256.6282587904207
11,Idaho,0.0012500764260188818,0.0012500764260188818,1.0,266.30201988198246
27,Nevada,0.0010417303550157347,0.0010417303550157347,1.0,213.33867737939858
40,South Dakota,0.0010417303550157347,0.0010417303550157347,1.0,246.01173707343767
22,Minnesota,0.0008333842840125878,0.0008333842840125878,1.0,217.51076181757654
30,New Mexico,0.0008333842840125878,0.0008333842840125878,1.0,207.5011737972915
48,Wisconsin,0.0006250382130094409,0.0006250382130094409,1.0,212.18622319581146
25,Montana,0.0006250382130094409,0.0006250382130094409,1.0,264.78240374339936
43,Utah,0.0004166921420062939,0.0004166921420062939,1.0,233.18645428200213
26,Nebraska,0.0004166921420062939,0.0004166921420062939,1.0,247.88735467055173
\end{filecontents*}
\csvreader{data/bounds_alpha2.csv}{1=\index,2=\statename,3=\lowerbound,4=\lowerboundbaseline,5=\upperbound,6=\upperboundbaseline}{
        \node[baseline end point] (baseline left) at (\lowerboundbaseline,\thecsvrow) {};
        \pgfmathparse{\upperboundbaseline > 1.0 ? 1 : \upperboundbaseline}
        \node[baseline end point] (baseline right) at (\pgfmathresult,\thecsvrow) {};
        \draw[outer line] (baseline left) edge (baseline right);
        \node[end point] (left) at (\lowerbound,\thecsvrow) {};
        \node[end point] (right) at (\upperbound,\thecsvrow) {};
        \pgfmathparse{\upperboundbaseline > 0.8 ? 1 : 0}
        \ifthenelse{\pgfmathresult>0}{
          \draw[inner line] (left) edge node[state name,pos=0.5] {\statename} (right);
        }{
          \draw[inner line] (left) edge (right);
          \node[right=0.125pc of baseline right,state name] {\statename};
        }
      }
      \node[] at (0.5, 52) {\Large $\alpha = 2$};

    \end{tikzpicture}}
    \quad
    \scalebox{0.75}{
      \begin{tikzpicture}[x=0.4\textwidth,y=2.6mm,yscale=-1]
      \tikzset{
        end point/.style = {fill=black, inner sep=0.125pc, circle},
        baseline end point/.style = {fill=blue!25, inner sep=0.2pc, rectangle},
        inner line/.style = {line width=0.125pc,black},
        outer line/.style = {line width=0.25pc,blue!25},
        state name/.style = {fill=white, inner xsep=0.0625pc, inner ysep=0pc, text=black, font=\small}
      }
      \draw (0,0) -- (0,50) ;
      \draw (1,0) -- (1,50) ;
      \draw[very thin, font=\small] (0,0) grid[xstep=0.2, ystep=50] (1,50);
      \foreach \i in {0.0, 0.2, 0.4, 0.6, 0.8, 1.0} {
        \node[above=0pc of {\i,0}, gray, overlay] {\i};
      };
      
      \begin{filecontents*}{data/bounds_alpha3.csv}
,state,lower,lower_baseline,upper,upper_baseline
45,Virginia,0.556087849365429,0.4490901768731688,1.0,6.158731855660819
20,Massachusetts,0.5546521826409375,0.4655675862404577,1.0,6.678903478427855
31,New York,0.4797833904004717,0.4797833904004717,1.0,4.80906212547854
37,Pennsylvania,0.38576523106947014,0.38576523106947014,1.0,5.320824888938267
46,Washington,0.34260165402137366,0.04361667185458834,1.0,19.901919596796116
7,Delaware,0.21172573048806756,0.14765051138923607,1.0,9.305791386969968
4,California,0.17220508221107841,0.17220508221107841,1.0,6.57949233917663
42,Texas,0.0991875426619157,0.0991875426619157,1.0,5.588968780270646
44,Vermont,0.08901689526685291,0.014538890618196111,1.0,14.564194054883274
19,Maryland,0.08820260308372307,0.08820260308372307,1.0,7.767039046765628
47,West Virginia,0.07529787658805327,0.02423148436366019,1.0,13.759369558394642
34,Ohio,0.06828223969501909,0.011954198952739028,1.0,13.167736672520977
0,Alabama,0.06115933488044973,0.0171235822836532,1.0,15.10184182416876
18,Maine,0.053926683028686424,0.015185063534560386,1.0,16.75893950001527
16,Kentucky,0.05392668302868642,0.013246544785467568,1.0,14.720795920583456
17,Louisiana,0.051693833309141735,0.051693833309141735,1.0,7.889392355625725
6,Connecticut,0.05072457393459532,0.05072457393459532,1.0,8.515815919702986
28,New Hampshire,0.04743029319839837,0.02229296561456737,1.0,11.116991316815657
38,Rhode Island,0.0465817280364095,0.019062101032746012,1.0,16.329550606781595
32,North Carolina,0.03764931278678344,0.02552383019638873,1.0,9.042941358761334
9,Georgia,0.03764931278678344,0.026493089570935137,1.0,9.215881811955562
35,Oklahoma,0.031544281064828296,0.004200123956367766,1.0,19.631530350796275
13,Indiana,0.031544281064828296,0.004523210414549901,1.0,19.54159948157524
39,South Carolina,0.03101629998548504,0.03101629998548504,1.0,8.555038691455602
29,New Jersey,0.028111124427468093,0.01550814999274252,1.0,12.618100299277037
49,Wyoming,0.023846264113002898,0.001938518749092815,1.0,27.804275874528837
3,Arkansas,0.023846264113002898,0.0029077781236392223,1.0,21.505858551568394
10,Hawaii,0.023846264113002898,0.008077161454553396,1.0,24.127569300087135
8,Florida,0.02261605207274951,0.02261605207274951,1.0,12.150357840356735
21,Michigan,0.01602488730703683,0.004846296872732038,1.0,16.44622871709988
1,Alaska,0.01602488730703683,0.005815556247278445,1.0,29.07372426595461
5,Colorado,0.01602488730703683,0.004846296872732038,1.0,19.70131163256809
23,Mississippi,0.015831236450924657,0.015831236450924657,1.0,10.98378752249515
12,Illinois,0.014538890618196111,0.014538890618196111,1.0,11.453882887882758
41,Tennessee,0.008077161454553396,0.008077161454553396,1.0,12.695686781769991
36,Oregon,0.004846296872732038,0.004846296872732038,1.0,19.952053539414603
24,Missouri,0.004200123956367766,0.004200123956367766,1.0,14.291678601560921
14,Iowa,0.00387703749818563,0.00387703749818563,1.0,23.871614012365782
2,Arizona,0.0032308645818213584,0.0032308645818213584,1.0,24.369799420299252
15,Kansas,0.0022616052072749505,0.0022616052072749505,1.0,21.85855597954789
33,North Dakota,0.001938518749092815,0.001938518749092815,1.0,27.847354232619182
11,Idaho,0.001938518749092815,0.001938518749092815,1.0,28.897077490487085
27,Nevada,0.0016154322909106792,0.0016154322909106792,1.0,23.14989685276364
40,South Dakota,0.0016154322909106792,0.0016154322909106792,1.0,26.65708878746101
22,Minnesota,0.0012923458327285433,0.0012923458327285433,1.0,23.568809191687404
30,New Mexico,0.0012923458327285433,0.0012923458327285433,1.0,22.48420046627934
48,Wisconsin,0.0009692593745464075,0.0009692593745464075,1.0,22.991858268608937
25,Montana,0.0009692593745464075,0.0009692593745464075,1.0,28.691021533814634
43,Utah,0.0006461729163642717,0.0006461729163642717,1.0,25.303627220346325
26,Nebraska,0.0006461729163642717,0.0006461729163642717,1.0,26.860325045260648
\end{filecontents*}
\csvreader{data/bounds_alpha3.csv}{1=\index,2=\statename,3=\lowerbound,4=\lowerboundbaseline,5=\upperbound,6=\upperboundbaseline}{
        \node[baseline end point] (baseline left) at (\lowerboundbaseline,\thecsvrow) {};
        \pgfmathparse{\upperboundbaseline > 1.0 ? 1 : \upperboundbaseline}
        \node[baseline end point] (baseline right) at (\pgfmathresult,\thecsvrow) {};
        \draw[outer line] (baseline left) edge (baseline right);
        \node[end point] (left) at (\lowerbound,\thecsvrow) {};
        \node[end point] (right) at (\upperbound,\thecsvrow) {};
        \pgfmathparse{\upperboundbaseline > 0.8 ? 1 : 0}
        \ifthenelse{\pgfmathresult>0}{
          \draw[inner line] (left) edge node[state name,pos=0.5] {\statename} (right);
        }{
          \draw[inner line] (left) edge (right);
          \node[right=0.125pc of baseline right,state name] {\statename};
        }
      }
      \node[] at (0.5, 52) {\Large $\alpha = 3$};
    \end{tikzpicture}}
  \quad
    \scalebox{0.75}{
      \begin{tikzpicture}[x=0.4\textwidth,y=2.6mm,yscale=-1]
      \tikzset{
        end point/.style = {fill=black, inner sep=0.125pc, circle},
        baseline end point/.style = {fill=blue!25, inner sep=0.2pc, rectangle},
        inner line/.style = {line width=0.125pc,black},
        outer line/.style = {line width=0.25pc,blue!25},
        state name/.style = {fill=white, inner xsep=0.0625pc, inner ysep=0pc, text=black, font=\small}
      }
      \draw (0,0) -- (0,50) ;
      \draw (1,0) -- (1,50) ;
      \draw[very thin, font=\small] (0,0) grid[xstep=0.2, ystep=50] (1,50);
      \foreach \i in {0.0, 0.2, 0.4, 0.6, 0.8, 1.0} {
        \node[above=0pc of {\i,0}, gray, overlay] {\i};
      };
      
      \begin{filecontents*}{data/bounds_alpha4.csv}
,state,lower,lower_baseline,upper,upper_baseline
45,Virginia,0.585831046038332,0.4775005068658064,1.0,3.564267366502117
20,Massachusetts,0.5844196845838789,0.495020309635703,1.0,3.856700407116964
31,New York,0.510135433594045,0.510135433594045,1.0,2.456711414565766
37,Pennsylvania,0.4101695001422827,0.4101695001422827,1.0,2.525068586487641
46,Washington,0.35666584250183253,0.04637594850854954,1.0,3.860961265918559
7,Delaware,0.2244420447713011,0.1569911738400529,1.0,2.1992941621104185
4,California,0.1830991152226438,0.1830991152226438,1.0,1.3899837198842697
42,Texas,0.10546234216388675,0.10546234216388675,1.0,1.0377781701768163
44,Vermont,0.09416220463228393,0.01545864950284985,1.0,1.8617951230979002
19,Maryland,0.09378247365062241,0.09378247365062241,1.0,1.404697867186555
47,West Virginia,0.07971959939886995,0.025764415838083084,1.0,1.7209016142857014
34,Ohio,0.07232420512492302,0.012710445146787654,1.0,1.628631083046134
0,Alabama,0.06480898884472544,0.018206853858912044,1.0,1.8468228578493986
18,Maine,0.05717101469584282,0.0161457005918654,1.0,2.0260650642475557
16,Kentucky,0.05717101469584281,0.014084547324818751,1.0,1.779664538593561
17,Louisiana,0.054964087121243906,0.054964087121243906,1.0,1.2798247330112955
6,Connecticut,0.05393351048772058,0.05393351048772058,1.0,1.2153840138950691
28,New Hampshire,0.05037684226412808,0.023703262571036433,1.0,1.451643369288814
38,Rhode Island,0.049407250113914385,0.02026800712595869,1.0,1.9512749669462939
32,North Carolina,0.04001375462289343,0.02713851801611418,1.0,1.1435221829101376
9,Georgia,0.04001375462289343,0.028169094649637502,1.0,1.1653913111842615
35,Oklahoma,0.03348971159494774,0.004465832078601068,1.0,2.290564771469322
13,Indiana,0.03348971159494774,0.004809357623108842,1.0,2.28007183091785
39,South Carolina,0.032978452272746345,0.032978452272746345,1.0,1.0581166890011502
29,New Jersey,0.029878862475658855,0.016489226136373172,1.0,1.4900451810488893
49,Wyoming,0.02532935186845115,0.0020611532670466465,1.0,3.204811717061472
3,Arkansas,0.02532935186845115,0.0030917299005699694,1.0,2.478835549703054
10,Hawaii,0.02532935186845115,0.008588138612694361,1.0,2.781022499779172
8,Florida,0.024046788115544213,0.024046788115544213,1.0,1.4176776406398874
21,Michigan,0.01703002104408502,0.005152883167616616,1.0,1.8723068994848404
1,Alaska,0.01703002104408502,0.006183459801139939,1.0,3.309873374208167
5,Colorado,0.01703002104408502,0.005152883167616616,1.0,2.242879041333361
23,Mississippi,0.01683275168088095,0.01683275168088095,1.0,1.3585127204161433
12,Illinois,0.01545864950284985,0.01545864950284985,1.0,1.3525651719828138
41,Tennessee,0.008588138612694361,0.008588138612694361,1.0,1.4272550495396563
36,Oregon,0.005152883167616616,0.005152883167616616,1.0,2.2145224348540347
24,Missouri,0.004465832078601068,0.004465832078601068,1.0,1.6066771968386477
14,Iowa,0.004122306534093293,0.004122306534093293,1.0,2.683658019094626
2,Arizona,0.003435255445077744,0.003435255445077744,1.0,2.704857795340948
15,Kansas,0.002404678811554421,0.002404678811554421,1.0,2.4573490929417443
33,North Dakota,0.0020611532670466465,0.0020611532670466465,1.0,3.1306125952867982
11,Idaho,0.0020611532670466465,0.0020611532670466465,1.0,3.2486229752027955
27,Nevada,0.001717627722538872,0.001717627722538872,1.0,2.6025222382512703
40,South Dakota,0.001717627722538872,0.001717627722538872,1.0,2.958729087765889
22,Minnesota,0.0013741021780310978,0.0013741021780310978,1.0,2.6159541229517527
30,New Mexico,0.0013741021780310978,0.0013741021780310978,1.0,2.4955710079651303
48,Wisconsin,0.0010305766335233233,0.0010305766335233233,1.0,2.55191706729515
25,Montana,0.0010305766335233233,0.0010305766335233233,1.0,3.184479769964416
43,Utah,0.0006870510890155489,0.0006870510890155489,1.0,2.844645614112528
26,Nebraska,0.0006870510890155489,0.0006870510890155489,1.0,2.9812867283407787
\end{filecontents*}
\csvreader{data/bounds_alpha4.csv}{1=\index,2=\statename,3=\lowerbound,4=\lowerboundbaseline,5=\upperbound,6=\upperboundbaseline}{
        \node[baseline end point] (baseline left) at (\lowerboundbaseline,\thecsvrow) {};
        \pgfmathparse{\upperboundbaseline > 1.0 ? 1 : \upperboundbaseline}
        \node[baseline end point] (baseline right) at (\pgfmathresult,\thecsvrow) {};
        \draw[outer line] (baseline left) edge (baseline right);
        \node[end point] (left) at (\lowerbound,\thecsvrow) {};
        \node[end point] (right) at (\upperbound,\thecsvrow) {};
        \pgfmathparse{\upperboundbaseline > 0.8 ? 1 : 0}
        \ifthenelse{\pgfmathresult>0}{
          \draw[inner line] (left) edge node[state name,pos=0.5] {\statename} (right);
        }{
          \draw[inner line] (left) edge (right);
          \node[right=0.125pc of baseline right,state name] {\statename};
        }
      }
      \node[] at (0.5, 52) {\Large $\alpha = 4$};
    \end{tikzpicture}}\\[5mm]
    \scalebox{0.75}{
  \begin{tikzpicture}[x=0.4\textwidth,y=2.6mm,yscale=-1]
      \tikzset{
        end point/.style = {fill=black, inner sep=0.125pc, circle},
        baseline end point/.style = {fill=blue!25, inner sep=0.2pc, rectangle},
        inner line/.style = {line width=0.125pc,black},
        outer line/.style = {line width=0.25pc,blue!25},
        state name/.style = {fill=white, inner xsep=0.0625pc, inner ysep=0pc, text=black, font=\small}
      }
      \draw (0,0) -- (0,50) ;
      \draw (1,0) -- (1,50) ;
      \draw[very thin, font=\small] (0,0) grid[xstep=0.2, ystep=50] (1,50);
      \foreach \i in {0.0, 0.2, 0.4, 0.6, 0.8, 1.0} {
        \node[above=0pc of {\i,0}, gray, overlay] {\i};
      };
      
      \begin{filecontents*}{data/bounds_alpha5.csv}
,state,lower,lower_baseline,upper,upper_baseline
45,Virginia,0.5893868466441953,0.48093228639321645,1.0,3.271622807514907
20,Massachusetts,0.5879790802055338,0.4985780033759892,1.0,3.5383679047345797
31,New York,0.5138017592042636,0.5138017592042636,1.0,2.1913762599491213
37,Pennsylvania,0.4131173740672665,0.4131173740672665,1.0,2.209719167143043
46,Washington,0.3583250285602504,0.046709250836751245,1.0,2.0516093953928336
7,Delaware,0.22597290913263723,0.15811946394366902,1.0,1.3977115021807691
4,California,0.18441504219250676,0.18441504219250676,0.8046292208733417,0.8046292208733417
42,Texas,0.10622029634727874,0.10622029634727874,0.5244232275583791,0.5244232275583791
44,Vermont,0.09478014868669438,0.01556975027891708,0.42901855243945014,0.42901855243945014
19,Maryland,0.09445648502543029,0.09445648502543029,0.687052843944077,0.687052843944077
47,West Virginia,0.08025116005258932,0.025949583798195133,0.3630136350488729,0.3630136350488729
34,Ohio,0.07281035559738355,0.012801794673776266,0.3270690534059778,0.3270690534059778
0,Alabama,0.06524817641174313,0.018337705884057896,0.3517130978297262,0.3517130978297262
18,Maine,0.05756162826976182,0.016261739180202284,0.36425949323846835,0.36425949323846835
16,Kentucky,0.057561628269761814,0.014185772476346673,0.3199599629853522,0.3199599629853522
17,Louisiana,0.055359112102816285,0.055359112102816285,0.5342936185563998,0.5342936185563998
6,Connecticut,0.05432112875088848,0.05432112875088848,0.3919263468432895,0.3919263468432895
28,New Hampshire,0.05073233888412049,0.023873617094339526,0.3614332317512924,0.3614332317512924
38,Rhode Island,0.049747617641796325,0.020413672587913507,0.3293614265361677,0.3294666318549302
32,North Carolina,0.04029922496705092,0.02733356160076554,0.25250129955177864,0.25250129955177864
9,Georgia,0.04029922496705092,0.028371544952693346,0.2573302249471986,0.2573302249471986
35,Oklahoma,0.03372431317352716,0.0044979278583538234,0.33457761281675297,0.33457761281675297
13,Indiana,0.03372431317352716,0.0048439223089964244,0.33304493273502467,0.33304493273502467
39,South Carolina,0.033215467261689766,0.033215467261689766,0.21249577480578535,0.21249577480578535
29,New Jersey,0.03009230842608477,0.016607733630844883,0.21261102354279576,0.23484789618829308
49,Wyoming,0.025508297332083023,0.0020759667038556103,0.43009679235727233,0.43009679235727233
3,Arkansas,0.025508297332083023,0.0031139500557834155,0.3326682853262951,0.3326682853262951
10,Hawaii,0.025508297332083023,0.008649861266065043,0.3732228168852149,0.3732228168852149
8,Florida,0.024219611544982127,0.024219611544982127,0.1776425726744549,0.207076964885266
21,Michigan,0.017151365599173674,0.005189916759639027,0.14067596694751472,0.22843050770851053
1,Alaska,0.017151365599173674,0.006227900111566831,0.40382057852229486,0.40382057852229486
5,Colorado,0.017151365599173674,0.005189916759639027,0.24943900565212637,0.27364210337606854
23,Mississippi,0.016953728081487485,0.016953728081487485,0.2667877810946918,0.272822663236559
12,Illinois,0.01556975027891708,0.01556975027891708,0.21261102354279576,0.21317949894252128
41,Tennessee,0.008649861266065043,0.008649861266065043,0.051237054848241564,0.15622387970262375
36,Oregon,0.005189916759639027,0.005189916759639027,0.21380437172138403,0.21380437172138403
24,Missouri,0.0044979278583538234,0.0044979278583538234,0.05892467759938142,0.17586299323363894
14,Iowa,0.004151933407711221,0.004151933407711221,0.165615182166497,0.29374670467849795
2,Arizona,0.003459944506426018,0.003459944506426018,0.2611445304985936,0.2611445304985936
15,Kansas,0.0024219611544982122,0.0024219611544982122,0.165615182166497,0.2689754779336067
33,North Dakota,0.0020759667038556103,0.0020759667038556103,0.2529919713840486,0.3426692696861357
11,Idaho,0.0020759667038556103,0.0020759667038556103,0.2529919713840486,0.35558639995069774
27,Nevada,0.001729972253213009,0.001729972253213009,0.165615182166497,0.2848657786869317
40,South Dakota,0.001729972253213009,0.001729972253213009,0.28565491310783003,0.28565491310783003
22,Minnesota,0.001383977802570407,0.001383977802570407,0.2525611928363818,0.2525611928363818
30,New Mexico,0.001383977802570407,0.001383977802570407,0.24093862543291614,0.24093862543291614
48,Wisconsin,0.0010379833519278052,0.0010379833519278052,0.24637863977841304,0.24637863977841304
25,Montana,0.0010379833519278052,0.0010379833519278052,0.3074503494572077,0.3074503494572077
43,Utah,0.0006919889012852035,0.0006919889012852035,0.20099799182062753,0.31136801678091663
26,Nebraska,0.0006919889012852035,0.0006919889012852035,0.28783277416478925,0.28783277416478925
\end{filecontents*}
\csvreader{data/bounds_alpha5.csv}{1=\index,2=\statename,3=\lowerbound,4=\lowerboundbaseline,5=\upperbound,6=\upperboundbaseline}{
        \node[baseline end point] (baseline left) at (\lowerboundbaseline,\thecsvrow) {};
        \pgfmathparse{\upperboundbaseline > 1.0 ? 1 : \upperboundbaseline}
        \node[baseline end point] (baseline right) at (\pgfmathresult,\thecsvrow) {};
        \draw[outer line] (baseline left) edge (baseline right);
        \node[end point] (left) at (\lowerbound,\thecsvrow) {};
        \node[end point] (right) at (\upperbound,\thecsvrow) {};
        \pgfmathparse{\upperboundbaseline > 0.8 ? 1 : 0}
        \ifthenelse{\pgfmathresult>0}{
          \draw[inner line] (left) edge node[state name,pos=0.5] {\statename} (right);
        }{
          \draw[inner line] (left) edge (right);
          \node[right=0.125pc of baseline right,state name] {\statename};
        }
      }
      \node[] at (0.5, 52) {\Large $\alpha = 5$};

    \end{tikzpicture}}
    \quad
    \scalebox{0.75}{
      \begin{tikzpicture}[x=0.4\textwidth,y=2.6mm,yscale=-1]
      \tikzset{
        end point/.style = {fill=black, inner sep=0.125pc, circle},
        baseline end point/.style = {fill=blue!25, inner sep=0.2pc, rectangle},
        inner line/.style = {line width=0.125pc,black},
        outer line/.style = {line width=0.25pc,blue!25},
        state name/.style = {fill=white, inner xsep=0.0625pc, inner ysep=0pc, text=black, font=\small}
      }
      \draw (0,0) -- (0,50) ;
      \draw (1,0) -- (1,50) ;
      \draw[very thin, font=\small] (0,0) grid[xstep=0.2, ystep=50] (1,50);
      \foreach \i in {0.0, 0.2, 0.4, 0.6, 0.8, 1.0} {
        \node[above=0pc of {\i,0}, gray, overlay] {\i};
      };
      
      \begin{filecontents*}{data/bounds_alpha6.csv}
,state,lower,lower_baseline,upper,upper_baseline
45,Virginia,0.5897364804031855,0.4812701370800875,1.0,3.2430382941183393
20,Massachusetts,0.5883290754051863,0.4989282500233137,1.0,3.5072742814312674
31,New York,0.514162700405705,0.514162700405705,1.0,2.1654592351871766
37,Pennsylvania,0.4134075853767082,0.4134075853767082,1.0,2.178916921124693
46,Washington,0.35848792006899916,0.04674206367324591,1.0,1.8748781314966603
7,Delaware,0.22612355896678807,0.15823054147165466,1.0,1.31941566657558
4,California,0.18454459213214863,0.18454459213214863,0.7474538078855659,0.7474538078855659
42,Texas,0.10629491516804809,0.10629491516804809,0.47428048345524204,0.47428048345524204
44,Vermont,0.09484094241108271,0.015580687891081968,0.28906986786450434,0.28906986786450434
19,Maryland,0.09452283987256395,0.09452283987256395,0.6169557478222539,0.6169557478222539
47,West Virginia,0.08030346133454223,0.025967813151803278,0.23037981035770508,0.23037981035770508
34,Ohio,0.07285819169908926,0.012810787821556284,0.19993695362776442,0.19993695362776442
0,Alabama,0.06529139405523234,0.01835058796060765,0.2056759218353157,0.2056759218353157
18,Maine,0.0576000683937161,0.01627316290846339,0.20194004410400965,0.20194004410400965
16,Kentucky,0.05760006839371609,0.014195737856319128,0.17738104355863568,0.17738104355863568
17,Louisiana,0.055398001390513665,0.055398001390513665,0.4614727050138876,0.4614727050138876
6,Connecticut,0.05435928886444154,0.05435928886444154,0.31149383582442947,0.31149383582442947
28,New Hampshire,0.05076733172661813,0.02389038809965902,0.2549452574340338,0.2549452574340338
38,Rhode Island,0.049781115142747766,0.020428013012751912,0.1710539750759235,0.1710539750759235
32,North Carolina,0.0403273272422908,0.027352763186566122,0.16546944608037348,0.16546944608037348
9,Georgia,0.0403273272422908,0.028391475712638257,0.16863394310182242,0.16863394310182242
35,Oklahoma,0.03374740457276711,0.004501087612979235,0.14352351861620832,0.14352351861620832
13,Indiana,0.03374740457276711,0.004847325121669946,0.14286604594076333,0.14286604594076333
39,South Carolina,0.0332388008343082,0.0332388008343082,0.1298984340799259,0.1298984340799259
29,New Jersey,0.03011332069058668,0.0166194004171541,0.11224454511761232,0.11224454511761232
49,Wyoming,0.025525911747281402,0.0020774250521442624,0.15907218758978486,0.15907218758978486
3,Arkansas,0.025525911747281402,0.0031161375782163938,0.1230380528963313,0.1230380528963313
10,Hawaii,0.025525911747281402,0.00865593771726776,0.13803723021267258,0.13803723021267258
8,Florida,0.02423662560834973,0.02423662560834973,0.08882965711449486,0.08882965711449486
21,Michigan,0.01716331088449721,0.0051935626303606566,0.06742234997846505,0.06786231909890791
1,Alaska,0.01716331088449721,0.0062322751564327875,0.11996734251168759,0.11996734251168759
5,Colorado,0.01716331088449721,0.0051935626303606566,0.08129381633165829,0.08129381633165829
23,Mississippi,0.01696563792584481,0.01696563792584481,0.16677619481297629,0.16677619481297629
12,Illinois,0.015580687891081968,0.015580687891081968,0.10188822755311791,0.10188822755311791
41,Tennessee,0.00865593771726776,0.00865593771726776,0.027016665838594085,0.0320739292391675
36,Oregon,0.0051935626303606566,0.0051935626303606566,0.018381116710959457,0.018381116710959457
24,Missouri,0.004501087612979235,0.004501087612979235,0.029883692403103134,0.03610598591906046
14,Iowa,0.004154850104288525,0.004154850104288525,0.050675845652604475,0.0603083922767189
2,Arizona,0.003462375086907104,0.003462375086907104,0.022451028736580635,0.022451028736580635
15,Kansas,0.002423662560834973,0.002423662560834973,0.050675845652604475,0.05522267442554668
33,North Dakota,0.0020774250521442624,0.0020774250521442624,0.06381130776311177,0.0703525602441285
11,Idaho,0.0020774250521442624,0.0020774250521442624,0.06381130776311177,0.07300454355722577
27,Nevada,0.001731187543453552,0.001731187543453552,0.050675845652604475,0.058485071844695376
40,South Dakota,0.001731187543453552,0.001731187543453552,0.024558227012010405,0.024558227012010405
22,Minnesota,0.0013849500347628417,0.0013849500347628417,0.021713104951073165,0.021713104951073165
30,New Mexico,0.0013849500347628417,0.0013849500347628417,0.02071389354017418,0.02071389354017418
48,Wisconsin,0.0010387125260721312,0.0010387125260721312,0.021181580602832456,0.021181580602832456
25,Montana,0.0010387125260721312,0.0010387125260721312,0.0264320168511923,0.0264320168511923
43,Utah,0.0006924750173814209,0.0006924750173814209,0.05931422486638739,0.06392617925365295
26,Nebraska,0.0006924750173814209,0.0006924750173814209,0.024745461341907023,0.024745461341907023
\end{filecontents*}
\csvreader{data/bounds_alpha6.csv}{1=\index,2=\statename,3=\lowerbound,4=\lowerboundbaseline,5=\upperbound,6=\upperboundbaseline}{
        \node[baseline end point] (baseline left) at (\lowerboundbaseline,\thecsvrow) {};
        \pgfmathparse{\upperboundbaseline > 1.0 ? 1 : \upperboundbaseline}
        \node[baseline end point] (baseline right) at (\pgfmathresult,\thecsvrow) {};
        \draw[outer line] (baseline left) edge (baseline right);
        \node[end point] (left) at (\lowerbound,\thecsvrow) {};
        \node[end point] (right) at (\upperbound,\thecsvrow) {};
        \pgfmathparse{\upperboundbaseline > 0.8 ? 1 : 0}
        \ifthenelse{\pgfmathresult>0}{
          \draw[inner line] (left) edge node[state name,pos=0.5] {\statename} (right);
        }{
          \draw[inner line] (left) edge (right);
          \node[right=0.125pc of baseline right,state name] {\statename};
        }
      }
      \node[] at (0.5, 52) {\Large $\alpha = 6$};
    \end{tikzpicture}}
  \quad
    \scalebox{0.75}{
      \begin{tikzpicture}[x=0.4\textwidth,y=2.6mm,yscale=-1]
      \tikzset{
        end point/.style = {fill=black, inner sep=0.125pc, circle},
        baseline end point/.style = {fill=blue!25, inner sep=0.2pc, rectangle},
        inner line/.style = {line width=0.125pc,black},
        outer line/.style = {line width=0.25pc,blue!25},
        state name/.style = {fill=white, inner xsep=0.0625pc, inner ysep=0pc, text=black, font=\small}
      }
      \draw (0,0) -- (0,50) ;
      \draw (1,0) -- (1,50) ;
      \draw[very thin, font=\small] (0,0) grid[xstep=0.2, ystep=50] (1,50);
      \foreach \i in {0.0, 0.2, 0.4, 0.6, 0.8, 1.0} {
        \node[above=0pc of {\i,0}, gray, overlay] {\i};
      };
      
      \begin{filecontents*}{data/bounds_alpha7.csv}
,state,lower,lower_baseline,upper,upper_baseline
45,Virginia,0.5897667859745224,0.4812994247931292,1.0,3.2405622404779844
20,Massachusetts,0.5883594123725415,0.4989586123215102,1.0,3.50458088315284
31,New York,0.5141939897969762,0.5141939897969762,1.0,2.163214245299866
37,Pennsylvania,0.4134327433115081,0.4134327433115081,1.0,2.176248762589799
46,Washington,0.35850203706120176,0.046744908163361465,1.0,1.8595692793478167
7,Delaware,0.22613661804188886,0.15824017059745327,1.0,1.3126335084858025
4,California,0.18455582260053083,0.18455582260053083,0.7425011474185806,0.7425011474185806
42,Texas,0.10630138374927385,0.10630138374927385,0.4699370082891096,0.4699370082891096
44,Vermont,0.09484621216822713,0.015581636054453823,0.27694720392066663,0.27694720392066663
19,Maryland,0.09452859206368651,0.09452859206368651,0.6108837826422692,0.6108837826422692
47,West Virginia,0.08030799499554804,0.025969393424089705,0.21889077572190002,0.21889077572190002
34,Ohio,0.07286233832615935,0.01281156742255092,0.18892449054698152,0.18892449054698152
0,Alabama,0.06529514035745038,0.018351704686356725,0.19302585932658772,0.19302585932658772
18,Maine,0.057603400578499026,0.01627415321242955,0.187879575158679,0.187879575158679
16,Kentucky,0.05760340057849901,0.014196601738502372,0.16503054286664828,0.16503054286664828
17,Louisiana,0.055401372638058034,0.055401372638058034,0.45516479675708377,0.45516479675708377
6,Connecticut,0.05436259690109445,0.05436259690109445,0.304526594194384,0.304526594194384
28,New Hampshire,0.050770365155864056,0.02389184195016253,0.24572103406509724,0.24572103406509724
38,Rhode Island,0.049784018893388224,0.0204292561602839,0.15733192110724908,0.15733192110724908
32,North Carolina,0.04032976336807463,0.027354427740041157,0.15793055484143398,0.15793055484143398
9,Georgia,0.04032976336807463,0.028393203477004744,0.16095087540350766,0.16095087540350766
35,Oklahoma,0.033749406289036825,0.004501361526842215,0.1269739912939476,0.1269739912939476
13,Indiana,0.033749406289036825,0.004847620105830078,0.12639233101573774,0.12639233101573774
39,South Carolina,0.03324082358283482,0.03324082358283482,0.12274367012258387,0.12274367012258387
29,New Jersey,0.030115142200412177,0.01662041179141741,0.10162437224544589,0.10162437224544589
49,Wyoming,0.025527438691759285,0.002077551473927176,0.13559543813632216,0.13559543813632216
3,Arkansas,0.025527438691759285,0.0031163272108907644,0.10487941948055149,0.10487941948055149
10,Hawaii,0.025527438691759285,0.008656464474696568,0.11766493560822613,0.11766493560822613
8,Florida,0.024238100529150392,0.024238100529150392,0.07858681432735207,0.07858681432735207
21,Michigan,0.01716434639459692,0.005193878684817941,0.053953548205512646,0.053953548205512646
1,Alaska,0.01716434639459692,0.006232654421781529,0.09537934870539584,0.09537934870539584
5,Colorado,0.01716434639459692,0.005193878684817941,0.06463218316880068,0.06463218316880068
23,Mississippi,0.016966670370405272,0.016966670370405272,0.15759021565901404,0.15759021565901404
12,Illinois,0.015581636054453823,0.015581636054453823,0.09224793199025624,0.09224793199025624
41,Tennessee,0.008656464474696568,0.008656464474696568,0.02131978651404171,0.02131978651404171
36,Oregon,0.005193878684817941,0.005193878684817941,0.0014531230360093185,0.0014531230360093185
24,Missouri,0.004501361526842215,0.004501361526842215,0.023999925482573802,0.023999925482573802
14,Iowa,0.004155102947854352,0.004155102947854352,0.04008745042608012,0.04008745042608012
2,Arizona,0.003462585789878627,0.003462585789878627,0.0017748707846341506,0.0017748707846341506
15,Kansas,0.002423810052915039,0.002423810052915039,0.036706934803901956,0.036706934803901956
33,North Dakota,0.002077551473927176,0.002077551473927176,0.046763885831906535,0.046763885831906535
11,Idaho,0.002077551473927176,0.002077551473927176,0.048526679459479624,0.048526679459479624
27,Nevada,0.0017312928949393136,0.0017312928949393136,0.038875475364728365,0.038875475364728365
40,South Dakota,0.0017312928949393136,0.0017312928949393136,0.001941455786166752,0.001941455786166752
22,Minnesota,0.0013850343159514508,0.0013850343159514508,0.001716534065032084,0.001716534065032084
30,New Mexico,0.0013850343159514508,0.0013850343159514508,0.0016375411974140373,0.0016375411974140373
48,Wisconsin,0.001038775736963588,0.001038775736963588,0.001674514296224028,0.001674514296224028
25,Montana,0.001038775736963588,0.001038775736963588,0.0020895886348272445,0.0020895886348272445
43,Utah,0.0006925171579757254,0.0006925171579757254,0.04249222114894264,0.04249222114894264
26,Nebraska,0.0006925171579757254,0.0006925171579757254,0.001956257635378793,0.001956257635378793
\end{filecontents*}
\csvreader{data/bounds_alpha7.csv}{1=\index,2=\statename,3=\lowerbound,4=\lowerboundbaseline,5=\upperbound,6=\upperboundbaseline}{
        \node[baseline end point] (baseline left) at (\lowerboundbaseline,\thecsvrow) {};
        \pgfmathparse{\upperboundbaseline > 1.0 ? 1 : \upperboundbaseline}
        \node[baseline end point] (baseline right) at (\pgfmathresult,\thecsvrow) {};
        \draw[outer line] (baseline left) edge (baseline right);
        \node[end point] (left) at (\lowerbound,\thecsvrow) {};
        \node[end point] (right) at (\upperbound,\thecsvrow) {};
        \pgfmathparse{\upperboundbaseline > 0.8 ? 1 : 0}
        \ifthenelse{\pgfmathresult>0}{
          \draw[inner line] (left) edge node[state name,pos=0.5] {\statename} (right);
        }{
          \draw[inner line] (left) edge (right);
          \node[right=0.125pc of baseline right,state name] {\statename};
        }
      }
      \node[] at (0.5, 52) {\Large $\alpha = 7$};
    \end{tikzpicture}}
\caption{Versions of \Cref{fig:swaps-bounds} (right) with $\alpha$ varying from $2$ to $7$, showing how our bounds change as the lower bound on average number of states considered grows from $\alpha k = 6$  to $\alpha k = 21$.}

\end{figure}

%%% Local Variables:
%%% mode: latex
%%% TeX-master: "neurips_2024"
%%% End:

%% file: consideration-aistats.bbl
\begin{thebibliography}{}

\bibitem[Abaluck and Adams-Prassl, 2021]{abaluck2021consumers}
Abaluck, J. and Adams-Prassl, A. (2021).
\newblock What do consumers consider before they choose? {I}dentification from asymmetric demand responses.
\newblock {\em The Quarterly Journal of Economics}, 136(3):1611--1663.

\bibitem[Aho et~al., 1972]{aho1972transitive}
Aho, A.~V., Garey, M.~R., and Ullman, J.~D. (1972).
\newblock The transitive reduction of a directed graph.
\newblock {\em SIAM Journal on Computing}, 1(2):131--137.

\bibitem[Alvo and Yu, 2014]{alvo2014statistical}
Alvo, M. and Yu, P.~L. (2014).
\newblock {\em Statistical Methods for Ranking Data}.
\newblock Springer, New York.

\bibitem[Ba{\c{s}}ar and Bhat, 2004]{bacsar2004parameterized}
Ba{\c{s}}ar, G. and Bhat, C. (2004).
\newblock A parameterized consideration set model for airport choice: an application to the {S}an {F}rancisco {B}ay area.
\newblock {\em Transportation Research Part B: Methodological}, 38(10):889--904.

\bibitem[Beggs et~al., 1981]{beggs1981assessing}
Beggs, S., Cardell, S., and Hausman, J. (1981).
\newblock Assessing the potential demand for electric cars.
\newblock {\em Journal of Econometrics}, 17(1):1--19.

\bibitem[Ben-Akiva and Boccara, 1995]{ben1995discrete}
Ben-Akiva, M. and Boccara, B. (1995).
\newblock Discrete choice models with latent choice sets.
\newblock {\em International Journal of Research in Marketing}, 12(1):9--24.

\bibitem[Biscarri et~al., 2018]{biscarri2018simple}
Biscarri, W., Zhao, S.~D., and Brunner, R.~J. (2018).
\newblock A simple and fast method for computing the {P}oisson binomial distribution function.
\newblock {\em Computational Statistics \& Data Analysis}, 122:92--100.

\bibitem[Bradley and Terry, 1952]{bradley1952rank}
Bradley, R.~A. and Terry, M.~E. (1952).
\newblock Rank analysis of incomplete block designs: {I.} {T}he method of paired comparisons.
\newblock {\em Biometrika}, 39(3/4):324--345.

\bibitem[Brown, 2011]{brown2011wasted}
Brown, D.~G. (2011).
\newblock How {I} wasted too long finding a concentration inequality for sums of geometric variables.
\newblock \url{https://cs.uwaterloo.ca/~browndg/negbin.pdf}.

\bibitem[Brown, 1976]{brown1976recall}
Brown, J., editor (1976).
\newblock {\em Recall and Recognition}.
\newblock John Wiley \& Sons, London.

\bibitem[Cattaneo et~al., 2020]{cattaneo2020random}
Cattaneo, M.~D., Ma, X., Masatlioglu, Y., and Suleymanov, E. (2020).
\newblock A random attention model.
\newblock {\em Journal of Political Economy}, 128(7):2796--2836.

\bibitem[Chintagunta and Nair, 2011]{chintagunta2011structural}
Chintagunta, P.~K. and Nair, H.~S. (2011).
\newblock Discrete-choice models of consumer demand in marketing.
\newblock {\em Marketing Science}, 30(6):977--996.

\bibitem[Cutolo, 2023]{fotgottenstates}
Cutolo, M. (2023).
\newblock The {U.S.} state everyone forgets when listing all 50.
\newblock \url{https://www.rd.com/article/most-forgotten-us-state/}. Accessed 10/9/2023.

\bibitem[Danaf et~al., 2019]{danaf2019online}
Danaf, M., Becker, F., Song, X., Atasoy, B., and Ben-Akiva, M. (2019).
\newblock Online discrete choice models: Applications in personalized recommendations.
\newblock {\em Decision Support Systems}, 119:35--45.

\bibitem[Fok et~al., 2012]{fok2012rank}
Fok, D., Paap, R., and Van~Dijk, B. (2012).
\newblock A rank-ordered logit model with unobserved heterogeneity in ranking capabilities.
\newblock {\em Journal of Applied Econometrics}, 27(5):831--846.

\bibitem[Gu et~al., 2012]{gu2012identifying}
Gu, B., Konana, P., and Chen, H.-W.~M. (2012).
\newblock Identifying consumer consideration set at the purchase time from aggregate purchase data in online retailing.
\newblock {\em Decision Support Systems}, 53(3):625--633.

\bibitem[Guiver and Snelson, 2009]{guiver2009bayesian}
Guiver, J. and Snelson, E. (2009).
\newblock Bayesian inference for {P}lackett-{L}uce ranking models.
\newblock In {\em International Conference on Machine Learning}, pages 377--384.

\bibitem[Hauser and Wernerfelt, 1990]{hauser1990evaluation}
Hauser, J.~R. and Wernerfelt, B. (1990).
\newblock An evaluation cost model of consideration sets.
\newblock {\em Journal of Consumer Research}, 16(4):393--408.

\bibitem[Hausman and Ruud, 1987]{hausman1987specifying}
Hausman, J.~A. and Ruud, P.~A. (1987).
\newblock Specifying and testing econometric models for rank-ordered data.
\newblock {\em Journal of Econometrics}, 34(1-2):83--104.

\bibitem[Horne et~al., 2005]{horne2005improving}
Horne, M., Jaccard, M., and Tiedemann, K. (2005).
\newblock Improving behavioral realism in hybrid energy-economy models using discrete choice studies of personal transportation decisions.
\newblock {\em Energy Economics}, 27(1):59--77.

\bibitem[Jagabathula et~al., 2023]{jagabathula2023demand}
Jagabathula, S., Mitrofanov, D., and Vulcano, G. (2023).
\newblock Demand estimation under uncertain consideration sets.
\newblock {\em Operations Research}.

\bibitem[Liu et~al., 2019]{liu2019learning}
Liu, A., Zhao, Z., Liao, C., Lu, P., and Xia, L. (2019).
\newblock Learning {P}lackett-{L}uce mixtures from partial preferences.
\newblock In {\em Proceedings of the AAAI Conference on Artificial Intelligence}, volume~33, pages 4328--4335.

\bibitem[Luce, 1959]{luce1959individual}
Luce, R.~D. (1959).
\newblock {\em Individual Choice Behavior: A Theoretical Analysis}.
\newblock Wiley, New York.

\bibitem[Manski, 1977]{manski1977structure}
Manski, C.~F. (1977).
\newblock The structure of random utility models.
\newblock {\em Theory and Decision}, 8(3):229.

\bibitem[Manzini and Mariotti, 2014]{manzini2014stochastic}
Manzini, P. and Mariotti, M. (2014).
\newblock Stochastic choice and consideration sets.
\newblock {\em Econometrica}, 82(3):1153--1176.

\bibitem[Maystre and Grossglauser, 2015]{maystre2015fast}
Maystre, L. and Grossglauser, M. (2015).
\newblock Fast and accurate inference of {P}lackett--{L}uce models.
\newblock {\em Advances in Neural Information Processing Systems}, 28.

\bibitem[McFadden, 1973]{mcfadden1973conditional}
McFadden, D. (1973).
\newblock Conditional logit analysis of qualitative choice behavior.
\newblock In Zarembka, P., editor, {\em Frontiers in Econometrics}, pages 105--142. Academic Press, New York.

\bibitem[Mitzenmacher and Upfal, 2017]{mitzenmacher2017probability}
Mitzenmacher, M. and Upfal, E. (2017).
\newblock {\em Probability and Computing: Randomization and Probabilistic Techniques in Algorithms and Data Analysis}.
\newblock Cambridge University Press, second edition.

\bibitem[Moe, 2006]{moe2006empirical}
Moe, W.~W. (2006).
\newblock An empirical two-stage choice model with varying decision rules applied to internet clickstream data.
\newblock {\em Journal of Marketing Research}, 43(4):680--692.

\bibitem[Nguyen and Zhang, 2023]{nguyen2023efficient}
Nguyen, D. and Zhang, A.~Y. (2023).
\newblock Efficient and accurate learning of mixtures of {Plackett-Luce} models.
\newblock In {\em Proceedings of the AAAI Conference on Artificial Intelligence}, volume~37, pages 9294--9301.

\bibitem[Palma, 2017]{palma2017improving}
Palma, M.~A. (2017).
\newblock Improving the prediction of ranking data.
\newblock {\em Empirical Economics}, 53:1681--1710.

\bibitem[Paszke et~al., 2019]{paszke2019pytorch}
Paszke, A., Gross, S., Massa, F., Lerer, A., Bradbury, J., Chanan, G., Killeen, T., Lin, Z., Gimelshein, N., Antiga, L., et~al. (2019).
\newblock {PyTorch}: An imperative style, high-performance deep learning library.
\newblock {\em Advances in Neural Information Processing Systems}, 32:8026--8037.

\bibitem[Plackett, 1975]{plackett1975analysis}
Plackett, R.~L. (1975).
\newblock The analysis of permutations.
\newblock {\em Journal of the Royal Statistical Society Series C: Applied Statistics}, 24(2):193--202.

\bibitem[Putnam et~al., 2018]{putnam2018collective}
Putnam, A.~L., Ross, M.~Q., Soter, L.~K., and Roediger~III, H.~L. (2018).
\newblock Collective narcissism: {Americans} exaggerate the role of their home state in appraising {US} history.
\newblock {\em Psychological Science}, 29(9):1414--1422.

\bibitem[Riedmiller and Braun, 1993]{riedmiller1993direct}
Riedmiller, M. and Braun, H. (1993).
\newblock A direct adaptive method for faster backpropagation learning: The {RPROP} algorithm.
\newblock In {\em IEEE International Conference on Neural Networks}, pages 586--591. IEEE.

\bibitem[Roberts and Lattin, 1991]{roberts1991development}
Roberts, J.~H. and Lattin, J.~M. (1991).
\newblock Development and testing of a model of consideration set composition.
\newblock {\em Journal of Marketing Research}, 28(4):429--440.

\bibitem[Saha and Gopalan, 2020]{saha2020pac}
Saha, A. and Gopalan, A. (2020).
\newblock From {PAC} to instance-optimal sample complexity in the {Plackett-Luce} model.
\newblock In {\em International Conference on Machine Learning}, pages 8367--8376. PMLR.

\bibitem[Seshadri et~al., 2020]{seshadri2020learning}
Seshadri, A., Ragain, S., and Ugander, J. (2020).
\newblock Learning rich rankings.
\newblock {\em Advances in Neural Information Processing Systems}, 33:9435--9446.

\bibitem[Shocker et~al., 1991]{shocker1991consideration}
Shocker, A.~D., Ben-Akiva, M., Boccara, B., and Nedungadi, P. (1991).
\newblock Consideration set influences on consumer decision-making and choice: Issues, models, and suggestions.
\newblock {\em Marketing Letters}, 2:181--197.

\bibitem[Simon, 1957]{simon1957models}
Simon, H.~A. (1957).
\newblock {\em Models of Man: Social and Rational}.
\newblock Wiley, New York.

\bibitem[Suh, 2009]{suh2009role}
Suh, J.-C. (2009).
\newblock The role of consideration sets in brand choice: The moderating role of product characteristics.
\newblock {\em Psychology \& Marketing}, 26(6):534--550.

\bibitem[Thurner, 2000]{thurner2000empirical}
Thurner, P.~W. (2000).
\newblock The empirical application of the spatial theory of voting in multiparty systems with random utility models.
\newblock {\em Electoral Studies}, 19(4):493--517.

\bibitem[Train, 2009]{train2009discrete}
Train, K.~E. (2009).
\newblock {\em Discrete Choice Methods with Simulation}.
\newblock Cambridge University Press, Cambridge.

\bibitem[van Nierop et~al., 2010]{nierop10:_retriev_unobs_consid_sets_househ_panel_data}
van Nierop, E., Bronnenberg, B., Paap, R., Wedel, M., and Franses, P.~H. (2010).
\newblock Retrieving unobserved consideration sets from household panel data.
\newblock {\em Journal of Marketing Research}, 47(1):63--74.

\bibitem[Zhao et~al., 2016]{zhao2016learning}
Zhao, Z., Piech, P., and Xia, L. (2016).
\newblock Learning mixtures of {P}lackett-{L}uce models.
\newblock In {\em International Conference on Machine Learning}, pages 2906--2914.

\bibitem[Zhao and Xia, 2019]{zhao2019learning}
Zhao, Z. and Xia, L. (2019).
\newblock Learning mixtures of {Plackett-Luce} models from structured partial orders.
\newblock {\em Advances in Neural Information Processing Systems}, 32.

\end{thebibliography}
